\documentclass{article}

\usepackage{arxiv}

\usepackage[utf8]{inputenc} 
\usepackage[T1]{fontenc}    
\usepackage{hyperref}       
\usepackage{url}            
\usepackage{booktabs}       
\usepackage{setspace}

\usepackage{amsfonts}       
\usepackage{nicefrac}       
\usepackage{microtype}      
\usepackage{lipsum}		
\usepackage{graphicx}
\usepackage{natbib}
\usepackage{doi}
\usepackage{amsmath}        
\usepackage{amssymb}       
\usepackage{amsthm}         
\usepackage{bm}             
\usepackage{ebgaramond}
\usepackage[cmintegrals,cmbraces]{newtxmath}
\DeclareSymbolFont{letters}{OML}{cmm}{m}{it}
\DeclareMathSymbol{0}{\mathalpha}{letters}{`0}
\DeclareMathSymbol{1}{\mathalpha}{letters}{`1}
\DeclareMathSymbol{2}{\mathalpha}{letters}{`2}
\DeclareMathSymbol{3}{\mathalpha}{letters}{`3}
\DeclareMathSymbol{4}{\mathalpha}{letters}{`4}
\DeclareMathSymbol{5}{\mathalpha}{letters}{`5}
\DeclareMathSymbol{6}{\mathalpha}{letters}{`6}
\DeclareMathSymbol{7}{\mathalpha}{letters}{`7}
\DeclareMathSymbol{8}{\mathalpha}{letters}{`8}
\DeclareMathSymbol{9}{\mathalpha}{letters}{`9}
\renewcommand{\mathbf}{\bm} 
\usepackage[usenames,dvipsnames,svgnames,table]{xcolor}
\PassOptionsToPackage{hyphens}{url}
\usepackage{hyperref}       
\hypersetup{
	pdfborder={0 0 0},
	breaklinks=true
}
\usepackage{url}            
\usepackage{booktabs}       
\usepackage{nicefrac}       
\usepackage{cleveref}             
\usepackage{mathrsfs}
\usepackage{caption}
\usepackage{bbm}
\usepackage{xr-hyper}
\usepackage{mathtools}
\usepackage{floatpag}
\usepackage{nameref}
\usepackage{multirow}
\usepackage{subcaption}
\usepackage{enumerate}
\usepackage{enumitem}
\setlist[enumerate]{label={\itshape\arabic*.}} 
\usepackage{makecell}
\usepackage{booktabs}
\usepackage{relsize}
\usepackage[mathscr]{eucal}
\usepackage{array}
\usepackage{comment}
\usepackage{appendix}
\usepackage{url}
\usepackage{algorithm}
\usepackage{algorithmic}
\usepackage{wrapfig}
\usepackage{soul}
\usepackage{dsfont}
\usepackage[algo2e,ruled,vlined]{algorithm2e}

\theoremstyle{definition}

\crefname{equation}{equation}{equations}
\Crefname{equation}{Equation}{Equations}
\Crefname{figure}{Figure}{Figures}
\Crefname{chapter}{Appendix}{Appendices}

\usepackage{aliascnt}

\newtheorem{theorem}{Theorem}[section]
\Crefname{theorem}{Theorem}{Theorems}

\newaliascnt{proposition}{theorem}
\newtheorem{proposition}[proposition]{Proposition}
\aliascntresetthe{proposition}
\Crefname{proposition}{Proposition}{Propositions}

\newaliascnt{definition}{theorem}
\newtheorem{definition}[definition]{Definition}
\aliascntresetthe{definition}
\Crefname{definition}{Definition}{Definitions}

\newaliascnt{remark}{theorem}
\newtheorem{remark}[remark]{Remark}
\aliascntresetthe{remark}
\Crefname{remark}{Remark}{Remarks}

\newaliascnt{corollary}{theorem}
\newtheorem{corollary}[corollary]{Corollary}
\aliascntresetthe{corollary}
\Crefname{corollary}{Corollary}{Corollaries}

\newaliascnt{lemma}{theorem}

\aliascntresetthe{lemma}
\Crefname{lemma}{Lemma}{Lemmas}

\newaliascnt{result}{theorem}

\aliascntresetthe{result}
\Crefname{result}{Result}{Results}

\newaliascnt{example}{theorem}

\aliascntresetthe{example}
\crefname{example}{Example}{Examples}

\newaliascnt{claim}{theorem}
\aliascntresetthe{claim}
\crefname{claim}{Claim}{Claims}

\newaliascnt{note}{theorem}
\aliascntresetthe{note}
\crefname{note}{Note}{Notes}

\usepackage{tikz}
\usepackage{pgfplots}
\usepackage{pgfplotstable}
\pgfplotsset{compat=1.18}
\usetikzlibrary{arrows.meta, backgrounds, fillbetween, shapes.geometric, positioning}

\tikzstyle{startstop} = [rectangle, rounded corners, minimum width=3.5cm, minimum height=1cm,text centered, draw=black, fill=blue!10]
\tikzstyle{process} = [rectangle, minimum width=3.5cm, minimum height=1cm, text centered, draw=black, fill=orange!20]
\tikzstyle{arrow} = [thick,->,>=stealth]

\renewcommand{\le}{\leqslant}
\renewcommand{\ge}{\geqslant}

\DeclareMathOperator{\B}{\mathbf{B}}

\newcommand{\E}{\mathbb E}
\newcommand{\R}{\mathbb R}

\renewcommand{\S}{\mathbb S}

\newcommand{\K}{ K}
\renewcommand{\k}{ k}
\newcommand{\F}{{F}} 

\renewcommand{\Re}{{\mathbb R}}

\newcommand{\bPsi}{\mathbf \Psi}

\newcommand{\bLambda}{\mathbf \Lambda}
\newcommand{\bpsi}{\bm \psi}

\DeclareMathOperator{\minimize}{minimize}
\DeclareMathOperator{\maximize}{maximize}

\DeclareMathOperator{\tr}{tr}

\setcitestyle{authoryear,open={(},close={)}}
\numberwithin{equation}{section}
\title{\Large{K-Tensors: Clustering Symmetric Positive Semi-Definite Matrices}}

\date{}

\renewcommand{\shorttitle}{K-Tensors}

\newif\ifuniqueAffiliation
\uniqueAffiliationtrue

\newbox{\orcid}\sbox{\orcid}{\includegraphics[scale=0.06]{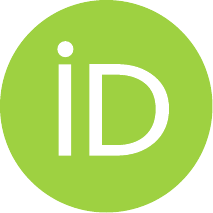}}

\ifuniqueAffiliation 
	\author{
		\hspace{1mm}Hanchao Zhang{\thanks {Corresponding author contact: \texttt{hanchao.zhang@nyu.edu}}} \\
		Division of Biostatistics\\
		Department of Population Health\\
		Grossman School of Medicine\\
		New York University\\
		 \\
		\And
		\hspace{1mm}Xiaomeng Ju \\
		Division of Biostatistics\\
		Department of Population Health\\
		Grossman School of Medicine\\
		New York University\\
		\And
		\hspace{1mm}Baoyi Shi\\
		Department of Biostatistics\\
		Columbia University\\
		\And
		\hspace{1mm}Lingsong Meng \\
		Division of Biostatistics\\
		Department of Population Health\\
		Weill Cornell Medicine\\
		Cornell University\\
		\And
		\hspace{1mm}Thaddeus Tarpey \\
		Division of Biostatistics\\
		Department of Population Health\\
		Grossman School of Medicine\\
		New York University\\
	}
\fi

\begin{document}
\maketitle

\begin{abstract}

This paper presents a new clustering algorithm for symmetric positive semi-definite (SPSD) matrices, called K-Tensors. The method identifies structured subsets of the SPSD cone characterized by common principal component (CPC) representations, where each subset corresponds to matrices sharing a common eigenstructure.  Unlike conventional clustering approaches that rely on vectorization  or transformations of SPSD matrices, thereby losing critical geometric and spectral information, K-Tensors introduces a divergence that respects the intrinsic geometry of SPSD matrices.  This divergence preserves the shape and eigenstructure information  and yields principal SPSD tensors, defined as a set of representative matrices that summarize the distribution of SPSD matrices. By exploring its theoretical properties, we show that the proposed clustering algorithm is self-consistent under mild distribution assumptions and converges to a local optimum. We demonstrate the use of the algorithm through an application to resting-state functional magnetic resonance imaging (rs-fMRI) data from the Human Connectome Project, where we cluster brain connectivity matrices to discover groups of subjects with shared connectivity structures.

\end{abstract}


\section{Introduction}
\label{intro}


The primary focus of this paper is to introduce a novel \textit{self-consistency algorithm} \citep{tarpey1999self} designed to cluster symmetric positive semi-definite (SPSD) matrices while preserving their {SPSD geometry}. This requires summarizing the distribution of the data while retaining the maximum amount of information, a central objective in statistics. Various statistical techniques have been developed for this purpose and are often prefixed with ``principal'', such as principal points, principal curves, and principal components. These techniques are frequently built upon the foundational concept of self-consistency.

The notion of self-consistency was first introduced by \citet{hastie1989principal} in Euclidean spaces as a smooth, nonlinear generalization of principal component axes, also referred to as principal curves. Consider \(\mathbf X \in \Re^p\) as a random vector with density \(h\) and finite second moments. Let \({f}\) denote a smooth, \(C^\infty\), unit-speed curve in \(\Re^p\). The projection index \(\lambda_{\mathbf{f}}: \Re^p \to \Re\) is defined as:
\[
	\lambda_{{f}}(\mathbf{x}) = \sup_\lambda  \left\{ \lambda \in \mathbb R: \Vert \mathbf{x} - {f}(\lambda) \Vert_2 = \inf_\mu \Vert \mathbf{x} - {f}(\mu) \Vert_2 \right\}.
\]
The projection index \(\lambda_{{f}}(\mathbf{x})\) represents the value of \(\lambda\) at which \({f}(\lambda)\) is closest to \(\mathbf{x}\). The curve \({f}\) is deemed self-consistent, or a principal curve of \(h\), if
\begin{align}\label{eq:principal_curve}
	\mathbb E(\mathbf X \mid \lambda_{{f}}(\mathbf X) = \lambda) = {f}(\lambda), \quad \text{for } a.s. \ \lambda.
\end{align}
\cite{flury1990principal}, \cite{tarpey1995principal}, and \citet{flury1996self} introduced the concept of an optimal \( K \)-point approximation to a distribution, referred to as \emph{principal points} or \emph{self-consistent points}. These points generalize the notion of the mean from a single representative point to multiple points that optimally summarize the distribution.

Let \( \bm{X} \in \mathbb{R}^p \) denote a random vector, and let \( \bm{Y} \in \mathbb{R}^p \) be a random vector, {measurable with respect to $\bm{X}$}, intended to approximate or summarize \( \bm{X} \). Then \( \bm{Y} \) is said to be \emph{self-consistent} for \( \bm{X} \) if
\begin{align}\label{eq:principal_points}
	\mathbb{E}[\bm{X} \mid \bm{Y}] = \bm{Y} \quad \text{a.s.}.
\end{align}
This property implies that \( \bm{Y} \) retains predictive information in \( \bm{X} \), in the sense that the conditional expectation of \( \bm{X} \) given \( \bm{Y} \) yields no improvement over \( \bm{Y} \) itself. In other words, \( \bm{Y} \) is a suitable mean-squared approximation of \( \bm{X} \) using only information contained in \( \bm{Y} \), making it an optimal summary under the squared error loss. \Cref{eq:principal_curve}, and \cref{eq:principal_points} are two layers of the definition of self-consistency. Whereas the self-consistency condition for principal points in \cref{eq:principal_points} involves taking the conditional expectation with respect to the random variable \(\bm{Y}\), the self-consistency condition for principal curves in \cref{eq:principal_curve} evaluates the conditional expectation at each point along the curve.

Although many statistical techniques for summarizing distributions were built upon the property of self-consistency, this property is not always preserved in the final estimation algorithm due to computational constraints or the risk of overfitting. For instance, \citet{meng2021principal} introduced principal manifolds through model complexity selection, demonstrating that enforcing self-consistency can lead to overfitting. In contrast, their model complexity selection procedure---despite relaxing the self-consistency condition---produced a smoother manifold that effectively mitigates overfitting.

Previously, the concept of self-consistency introduced above was formulated primarily for distributions defined on Euclidean spaces. However, as statistical science has evolved, the space containing data observations has broadened from Euclidean spaces to more general spaces, such as functional spaces. In this paper, the focus is on observations that reside in the space of SPSD matrices, often generated as covariance matrices from functional data. SPSD matrices serve as powerful data descriptors with applications across various domains, including computer vision \citep{tuzel2006region, alavi2014random, cherian2016positive}, brain imaging \citep{fmri, cherian2016positive}, and other machine learning fields \citep{shinohara2010covariance}. For example, in computer vision, an image patch can be represented by the covariance matrix of low-level features such as color and intensity gradients \citep{cherian2016positive}. Similarly, functional connectivity matrices derived from functional magnetic resonance imaging (fMRI) data quantify the temporal covariance between blood-oxygen-level-dependent (BOLD) signal time series across different brain regions \citep{fmri}. The BOLD signal reflects changes in blood oxygenation that occur in response to neural activity, serving as an indirect measure of brain function. These connectivity matrices enable the investigation of cognitive processes, neurological disorders, and developmental trajectories in the human brain. However, as with many complex data structures, the intricacy of random SPSD matrix-valued distributions presents challenges in preserving useful information while summarizing the distribution. Effectively capturing, summarizing, and preserving the information of SPSD matrices is crucial for developing clustering algorithms specifically tailored to the SPSD matrices.

An important feature of SPSD matrices is their eigenstructure, which can encapsulate key shape information. \citet{clark2022robust} demonstrated that eigenvector alignment captures the strength of bilateral connectivity in cortical areas of healthy control subjects, while also revealing degradation of this commissural system in individuals with Alzheimer's disease. Similarly, \citet{grigis2012longitudinal} showed that tests based on eigenvectors are effective in detecting the progression of brain lesions in patients with Neuromyelitis Optica and Multiple Sclerosis. Furthermore, \citet{schwartzman2010group} found that group-level differences in brain structure are more sensitively detected through differences in eigenvectors of diffusion tensors than through differences in eigenvalues or their combinations, underscoring the critical role of eigenstructure in symmetric positive definite data.

Given the central role that eigenstructure plays in these applications, the following discussion introduces the eigen-decomposition of SPSD matrices, along with several commonly used distance and divergence metrics for comparing such matrices, highlighting their relationship to the underlying eigenstructure. Let \(\bm \psi_i\) denote the covariance matrix for observation \(i=1,\dots n\). By performing principal component analysis (PCA) on each covariance matrix, we can decompose \(\bm \psi_i\) as \(\bm \psi_i = \mathbf U_i \mathbf D_i \mathbf U_i^\intercal\), where \(\mathbf U_i\) is an orthogonal matrix of eigenvectors and \(\mathbf D_i\) is the diagonal matrix of eigenvalues. The matrices $\bm{U}_i$ and $\bm{D}_i$ provide insight into the ``shape'' of the distribution for the observations, which can be visualized as ellipsoids whose axes are oriented according to the eigenvectors, with the eccentricity along each axis determined by the relative magnitudes of the eigenvalues. In particular, an SPSD matrix $\bm{\psi} \in \mathbb{R}^{p \times p}$ defines an ellipsoid via its quadratic form: for any $\bm{x} \in \mathbb{R}^p$, the level set $\left\{ \bm{x}\in \mathbb{R}^p : \bm{x}^\intercal \bm{\psi} \bm{x} = \mbox{constant} \right\}$ corresponds to an ellipsoid (possibly degenerate). Given the eigen-decomposition $\bm{\psi} = \bm{U} \bm{D} \bm{U}^\intercal$, where $\bm{D} = \mathrm{diag}(d_1, \ldots, d_p)$ and $\bm{U} = [\bm{u}_1, \ldots, \bm{u}_p]$, the quadratic form becomes $\bm{x}^\intercal \bm{\psi} \bm{x} = \sum_{j=1}^p d_j (\bm{u}_j^\intercal \bm{x})^2$, revealing that the ellipsoid is aligned with the eigenvectors $\bm{u}_j$ and scaled along each direction by $\sqrt{d_j}$. However, most clustering algorithms for SPSD matrices typically employ transformations that compromise this inherent eigenstructure. A common approach, particularly in the applied research domain, is to vectorize SPSD matrices and apply the conventional $K$-means algorithm directly on the vectorized SPSD matrices \citep{kauppi2010clustering, shakil2014cluster, huang2020statistical, fang2025individualized}. This approach typically relies on the Frobenius norm:
\[
	d_{\operatorname{E}}(\bm{\psi}_1, \bm{\psi}_2) = \Vert \bm{\psi}_1 - \bm{\psi}_2 \Vert_F^2 = \Vert \mathrm{vec}(\bm{\psi}_1) - \mathrm{vec}(\bm{\psi}_2) \Vert^2_2,
\]
which flattens each matrix into a vector and treats all entries independently. As a result, this method fails to preserve the eigenstructure information encoded in the eigenvectors and the scale information captured by the eigenvalues.

To better capture the geometry of SPSD matrices, \citet{stanitsas2017clustering} introduced the \(\alpha\beta\)-log-determinant divergence, a divergence family that includes both the log-determinant divergence and the affine-invariant Riemannian (AIR) distance as special cases. The log-determinant divergence \citep{springs2005matrix}, defined as
\[
	d_{\text{logdet}}(\bm{\psi}_1, \bm{\psi}_2) = \mathrm{tr}(\bm{\psi}_1 \bm{\psi}_2^{-1} - \mathbf{I}) - \log\det(\bm{\psi}_1 \bm{\psi}_2^{-1}),
\]
is particularly suited for comparing matrices that follow a Wishart distribution. Importantly, it satisfies the self-consistency property, as shown in \citet{banerjee2005clustering}, wherein the representative point minimizes the expected divergence conditioned on itself. However, this divergence is only weakly sensitive to differences in eigenstructure and thus may be limited in capturing orientation-based discrepancies. The affine-invariant Riemannian distance \citep{doi:10.1137/S0895479803436937}, defined as
\[
	d_{\text{AIR}}(\bm{\psi}_1, \bm{\psi}_2) = \Vert \log(\bm{\psi}_1^{-\frac{1}{2}} \bm{\psi}_2 \bm{\psi}_1^{-\frac{1}{2}}) \Vert_F,
\]
offers theoretical guarantees and respects the underlying manifold structure of SPSD matrices. However, its high computational cost—due to repeated evaluations of matrix logarithms and square roots—makes it impractical in high-dimensional settings. Additionally, the AIR distance does not satisfy the self-consistency property, complicating the derivation of closed-form expressions for cluster centers.
Several alternative approaches based on different distance metrics have also been proposed. For instance, \citet{nasim2021geometric} utilized the log-Euclidean distance, which approximates manifold geometry with improved efficiency, to cluster SPSD matrices. \citet{fryer2020k} proposed the log-Cholesky distance and demonstrated that the corresponding Fréchet mean under this metric coincides with the Euclidean mean. As a result, their clustering procedure reduces to applying the $K$-means algorithm to the vectorized upper-triangular entries of the Cholesky decomposition. While these methods offer interesting alternatives, a detailed comparison with them is beyond the scope of this paper, which focuses on developing and analyzing a clustering method with a novel divergence that preserves the eigenstructure and satisfies the self-consistency property with some distributional assumption.

Given the importance of eigenstructure in clustering SPSD matrices, along with the computational and statistical benefits of self-consistency, we propose a clustering algorithm that captures eigenstructural information through a novel divergence metric. This metric is shown to satisfy a self-consistency property under probabilistic assumptions. To develop a clustering algorithm for SPSD matrices, \Cref{sec:principal} introduces several fundamental concepts: (\Cref{sec2.1}), common principal components (\Cref{sec2.2}), domains of attraction (\Cref{sec2.3}), and the notions of principal SPSD tensors and self-consistency for random SPSD matrices (\Cref{sec2.4}). These concepts form the theoretical basis for the clustering algorithm presented in \Cref{sec3}, which leverages the unique structure of SPSD matrices to produce meaningful clustering results.

\section{Principal and Self-Consistent SPSD Matrices}
\label{sec:principal}


This section introduces a novel divergence metric tailored to capture the eigenstructure of SPSD matrices. Building on this divergence, the concept of \emph{principal SPSD tensors} is proposed as a structured summary of random SPSD matrices. To characterize desirable summaries under this metric, the notion of \emph{self-consistency} within the SPSD cone is also introduced; the proof of self-consistency is presented in \Cref{thm:self-consistency_spsd_tensors}.

The following notations will be used in the subsequent sections: Let $\mathcal{V}_q(\Re^p) = \{\bm{\Phi} \in \Re^{p \times q}: \bm{\Phi}^\intercal \bm{\Phi} = \mathbf{I}_q\}$ denote the Stiefel manifold, which consists of all orthonormal \(q\)-frames in \(\Re^p\). An orthonormal \(q\)-frame in \(\Re^p\) refers to a collection of \(q\) orthonormal vectors in \(\Re^p\), represented as the columns of a matrix with orthonormal columns. Two additional sets are considered: $\mathbb{S}_+^p = \{ \bm{\Sigma} \in \Re^{p \times p} \mid \bm{\Sigma} = \bm{\Sigma}^\intercal,\ \bm{\Sigma} \succeq 0 \}$, the cone of symmetric positive semi-definite (SPSD) matrices in $\Re^{p \times p}$; and $\mathbb{D}_+^p = \{ \bm{D} \in \Re^{p \times p} \mid \bm{D} = \operatorname{diag}(\bm{a}),\ \bm{a} = (a_1,\dots, a_p)^\intercal  \in \Re^p,\ a_j\ge 0,\;  j=1\dots, p\}$,
the set of diagonal matrices with non-negative entries. Here, $\mathbf{I}_q$ denotes the $q \times q$ identity matrix, and \(\operatorname{diag}(\bm{a})\) forms a diagonal matrix from the vector \(\bm{a}\).

\subsection{Projection Indices and Projection of SPSD Matrices}\label{sec2.1}

\begin{definition}\label{sec3:def:projectionindex}
	For a given orthogonal basis matrix $\mathbf B \in \mathcal V_p(\R^p)$, the projection index of a given SPSD matrix $\bpsi \in \S_+^p$,
	denoted $\bLambda_{\mathbf B}(\bm \psi): \S_+^p \to \mathbb D_+^p$,
	is defined as:
	\begin{align}\label{sec2:eq:projectionlambda}
		\bLambda_{\mathbf B}(\bpsi) = \arg\min_{\Theta \in \mathbb D_+^p}
		\Vert \bpsi - \mathbf B \mathbf \Theta \mathbf B^\intercal \Vert_F^2.
	\end{align}
\end{definition}
The goal is to find the non-negative diagonal matrix that provides the closest approximation to $\bpsi$ in terms of the Frobenius norm when the orthogonal basis matrix \(\bm B\) is fixed, the obtained \(\bm\Lambda_{\bm B}(\bm\psi)\) is used to construct the projection of a SPSD matrix onto the orthonormal matrix $\mathbf B$ in \Cref{def:projection}.
The following notation \(\odot\) represents the Hadamard (element-wise) product and is used below.
\begin{definition}\label{def:projection}
	For a given $\bpsi$ in $\S_+^p$ and a given orthogonal basis matrix $\mathbf B \in \mathcal V_p(\Re^p)$, {we define} the projection of the matrix $\bpsi$ onto $\mathbf B$ as:
	\begin{align*}
		\mathcal P_{\mathbf B}(\bpsi) = \mathbf B  \bm\Lambda_{\mathbf B}(\bpsi)   \mathbf B^\intercal.
	\end{align*}
\end{definition}
This projection allows us to approximate $\bpsi$ using a projection that is determined by an orthonormal matrix $\mathbf B$. \Cref{fig:ellipse_demos} illustrates the projection of a SPSD matrix onto two distinct orthogonal matrices.

\begin{figure}[htp]
	\centering
	\begin{subfigure}[t]{0.49\linewidth}
		\centering
		\includegraphics[width=0.95\linewidth]{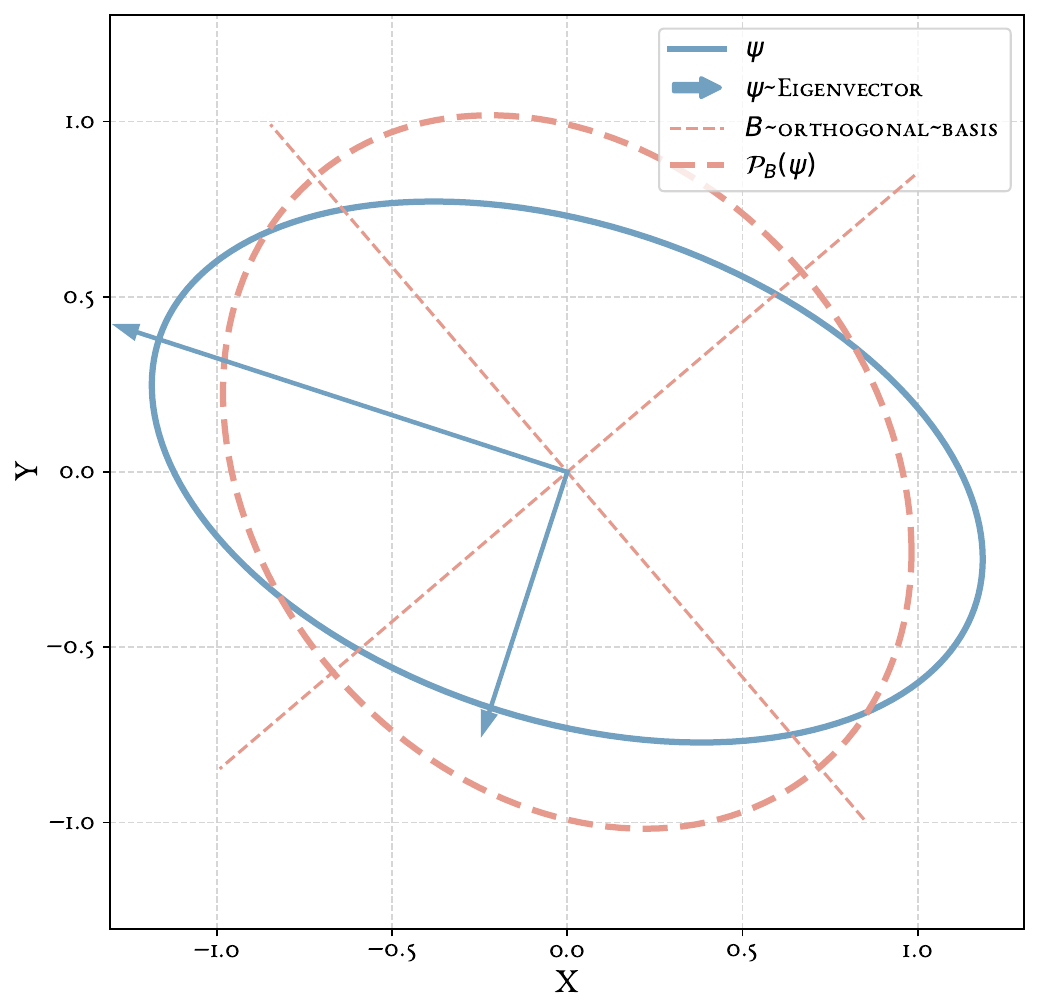}
		\caption*{\scriptsize \(\left\Vert \bm\psi - \mathcal P_{\bm B}(\bm\psi) \right\Vert_F^2 \approx 0.63\)}
	\end{subfigure}%
	\begin{subfigure}[t]{0.49\linewidth}
		\centering
		\includegraphics[width=0.95\linewidth]{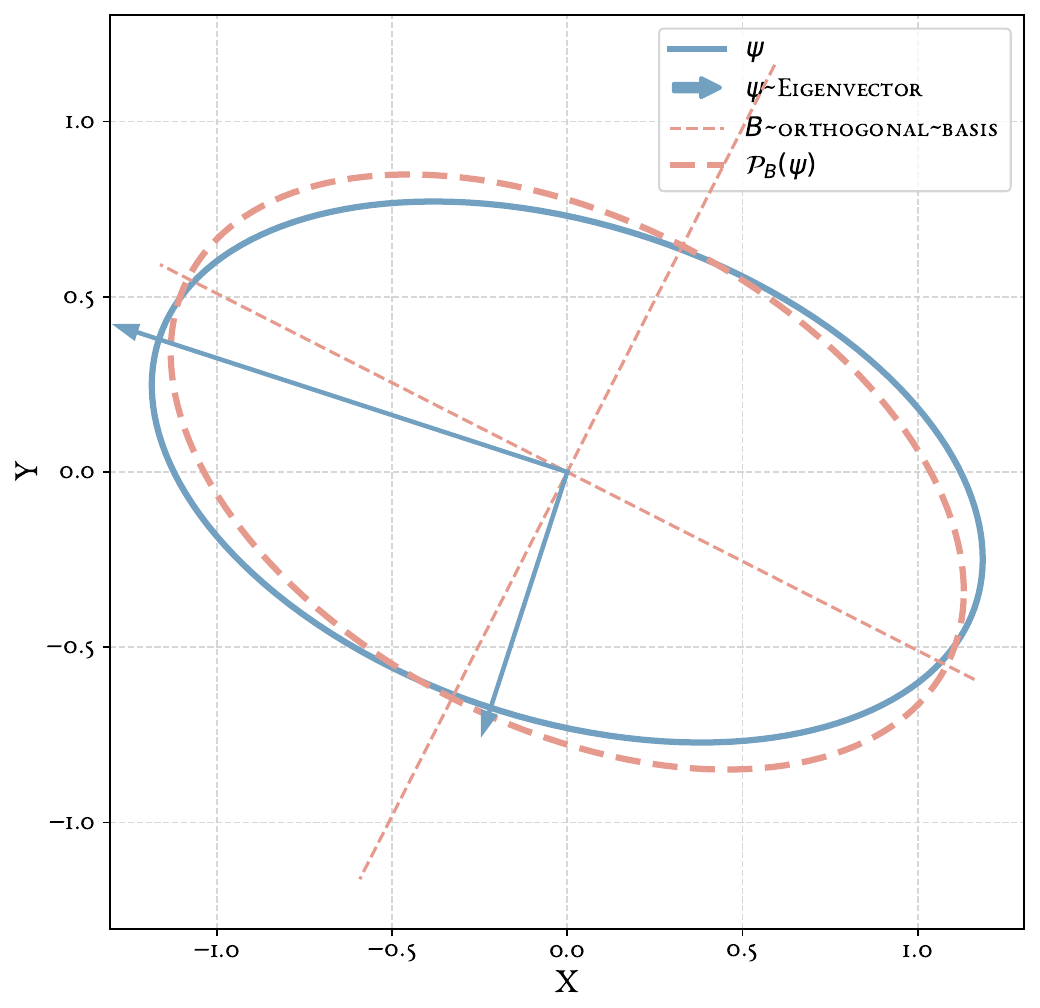}
		\caption*{\scriptsize \(\left\Vert \bm\psi - \mathcal P_{\bm B}(\bm\psi) \right\Vert_F^2 \approx 0.22\)}
	\end{subfigure}
	\caption{Illustration of projecting a symmetric positive semi-definite (SPSD) matrix \(\bm{\psi}\) onto an orthogonal matrix subspace defined by \(\bm{B}\). The squared Frobenius norm \(\left\Vert \bm{\psi} - \mathcal{P}_{\bm{B}}(\bm{\psi}) \right\Vert_F^2\) measures the projection error. Left: example with larger error (\(\approx 0.63\)), indicating that the subspace \(\bm{B}\) poorly captures the structure of \(\bm{\psi}\). Right: example with smaller error (\(\approx 0.22\)), where \(\bm{B}\) more accurately preserves the geometry of \(\bm{\psi}\).}
	\label{fig:ellipse_demos}
\end{figure}
\begin{theorem}\label{sec3:thm1}
	For a given orthogonal basis matrix $\mathbf B \in \mathcal V_p(\Re^p)$ and a SPSD matrix $\bpsi \in \mathbb S_+^p$, the projection index $\bm\Lambda_{\mathbf B}(\bpsi)$ that minimizes the \cref{sec2:eq:projectionlambda} is equal to the diagonal of the matrix $\mathbf B^\intercal \bm\psi \mathbf B$, that is:
	\begin{align}\label{eq:solution_lambda}
		\bm\Lambda_{\mathbf B}(\bm\psi) =  (\mathbf B^\intercal \bm\psi \mathbf B) \odot  \mathbf I,
	\end{align}
	assuming the eigenvalues of \(\bpsi\) are unique.
\end{theorem}
\begin{proof}
	Let $\bpsi = \mathbf U \mathbf D^* \mathbf U^\intercal$ be the spectral
	decomposition of $\bpsi$, and let $\mathbf B$ be a given orthonormal
	matrix. Our goal is to find $\mathbf D \in \mathbb D_+^p$ that minimizes
	the quantity:
	\[
		\underset{\mathbf D \in \mathbb D_+^p}{\min}
		\Vert \bpsi - \mathbf B \mathbf D \mathbf B^\intercal \Vert_\F^2.
	\]

	We can express the objective function as:
	\begin{align*}
		\Vert \bpsi - \mathbf B \mathbf D \mathbf B^\intercal \Vert_\F^2
		 & = \Vert \mathbf U \mathbf D^* \mathbf U^\intercal
		- \mathbf B \mathbf D \mathbf B^\intercal \Vert_\F^2                                                                                                                \\
		 & = \tr\Big\{
		\left( \mathbf U \mathbf D^* \mathbf U^\intercal
		- \mathbf B \mathbf D \mathbf B^\intercal \right)^\intercal
		\left( \mathbf U \mathbf D^* \mathbf U^\intercal
		- \mathbf B \mathbf D \mathbf B^\intercal \right)
		\Big\}                                                                                                                                                              \\
		 & = \tr\left\{
		\mathbf U \mathbf D^{*^2} \mathbf U^\intercal
		- 2\bpsi \mathbf B \mathbf D \mathbf B^\intercal
		+ \mathbf B \mathbf D^2 \mathbf B^\intercal
		\right\}                                                                                                                                                            \\
		 & = \tr\left\{ \mathbf U \mathbf D^{*^2} \mathbf U^\intercal - 2\mathbf B^\intercal \bpsi \mathbf B \mathbf D + \mathbf B \mathbf D^2 \mathbf B^\intercal \right\}
	\end{align*}

	Let $d_1^* \ge d_2^* \ge \cdots \ge d_p^*$ denote the ordered eigenvalues of $\bpsi$, and $d_1 \ge d_2 \ge \cdots \ge d_p$ denote the diagonal elements of $\mathbf D$, the objective function can be simplifies to:
	\begin{align}\label{sec2:eq:sumdj}
		\sum_{j=1}^p \left( d_j^{*2} - 2 a_j d_j + d_j^2 \right),
	\end{align}
	where $a_j$ are the diagonal elements of $\mathbf B^\intercal \bpsi \mathbf B$.
	The theorem is proved by noting that the individual terms in this summation, as a function of the $d_j$,  are minimized when $d_j=a_j$.
\end{proof}

\subsection{Common Principal Components (CPC)} \label{sec2.2}
The previously defined projection index $\bm\Lambda_{\mathbf B}(\bpsi)$ is based on a given orthogonal basis matrix $\mathbf B$ and a given observation $\bpsi$. A natural question that arises is which orthogonal basis matrix best characterizes the distribution of a random matrix $\bPsi \in \mathbb S_+^p$. To address this, the notion of common principal components (CPC) for random symmetric positive semi-definite (SPSD) matrices is introduced in \Cref{sec3:def:cpc}. \Cref{sec4:prop_equiv} establishes an equivalence between minimizing \(\E \left[ \left\Vert \bPsi - \mathcal P_{\mathbf B}\left( \bPsi \right) \right\Vert_{\F}^2 \right]\), and maximizing \( \E \left[\bm{\Lambda}_{\bm{B}}(\bm{\Psi})\right] \). This equivalence provides an alternative perspective for estimating CPCs. This viewpoint leads to \Cref{prop:cpc_fg_algo}, where the stationary condition for CPCs is derived via a least-squares formulation. While \citet{flury1984common} provided the CPC stationary conditions under the assumption that \( \bm{\Psi} \) follows a Wishart distribution, \Cref{prop:cpc_fg_algo} generalizes the result by framing the estimation as a least-squares problem. This formulation can also be solved using the Flury–Gautschi (FG) algorithm proposed in \citet{flury1984common}.

\begin{definition}\label{sec3:def:cpc}
	The common principal components (CPCs) of a random SPSD matrix $\bPsi$ is defined as follows:
	\begin{align}\label{sec3:cpcestimation}
		\underset{\mathbf B \in \mathcal V_p(\Re^p)}{\arg\min} \ \E \left\{ \left\Vert \bPsi - \mathcal P_{\mathbf B}\left( \bPsi \right) \right\Vert_{\F}^2 \right\},
	\end{align}
	The solution is an orthogonal basis matrix whose columns represent the common eigenvectors that diagonalize  \( \bm{\Psi} \) by its distribution.
\end{definition}
\begin{figure}[htp]
	\centering
	\begin{subfigure}[t]{0.49\linewidth}
		\centering
		\includegraphics[width=0.95\linewidth]{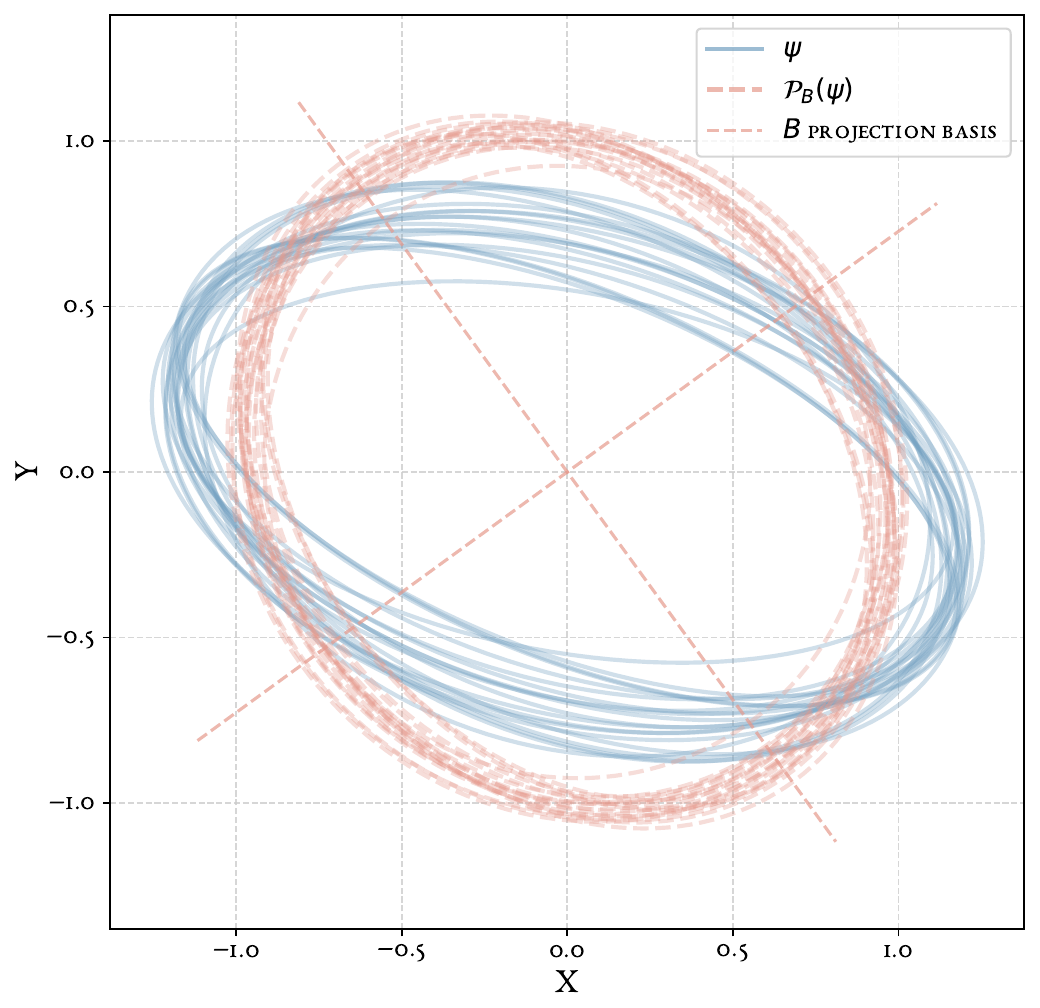}
		\caption*{\scriptsize  \(\sum_{i} \left\Vert \bm\psi_i - \mathcal P_{\bm B}(\bm\psi_i) \right\Vert_F^2 \approx 0.7\)}
	\end{subfigure}%
	\begin{subfigure}[t]{0.49\linewidth}
		\centering
		\includegraphics[width=0.95\linewidth]{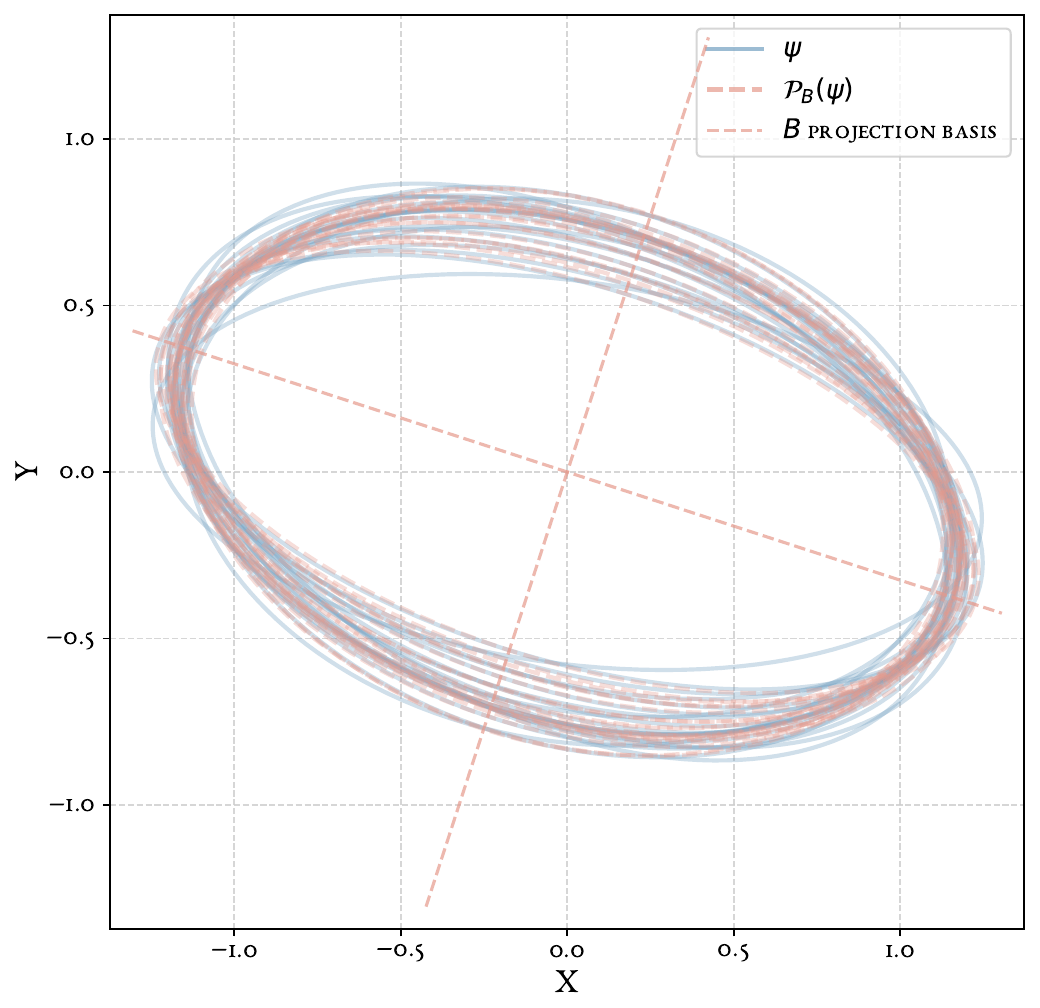}
		\caption*{\scriptsize \(\sum_{i} \left\Vert \bm\psi_i - \mathcal P_{\bm B}(\bm\psi_i) \right\Vert_F^2 \approx 0.13\)}
	\end{subfigure}
	\caption{Illustration of Common Principal Components (CPC). The blue ellipses represent the original matrices \(\bm{\psi}\), and the dashed pink ellipses represent their projections onto an orthogonal matrix subspace \(\mathcal{P}_{\bm{B}}(\bm{\psi})\). Dashed lines indicate the axes of the orthogonal matrix. Left: projection using an arbitrary orthogonal matrix, resulting in a larger expected projection error (\(\approx 0.70\)). Right: projection using the CPC matrix, which aligns the subspace with the common eigenstructure across matrices and achieves a smaller expected projection error (\(\approx 0.13\)), indicating better preservation of the original matrices' structure.}
	\label{fig:ellipse_demos_cpc}
\end{figure}
\Cref{fig:ellipse_demos_cpc} provides a graphical illustration of SPSD matrices and their projections onto both a random orthogonal matrix and the CPCs. The figure demonstrates how the choice of orthogonal basis affects the projection error, with the CPC basis yielding a smaller expected error and better preserving the eigenstructure of the original matrices.
\begin{proposition}\label{sec4:prop_equiv}
	Let \(\bPsi \in \S_+^p\) be a random SPSD matrix and \(\mathcal{P}_{\mathbf{B}}(\bPsi) = \mathbf{B} \bm{\Lambda}_{\mathbf{B}}(\bPsi) \mathbf{B}^\intercal\) denote its projection onto a subspace defined by orthogonal basis matrix \(\mathbf{B} \in \mathcal{V}_p(\Re^p)\), where \(\bm{\Lambda}_{\mathbf{B}}(\bPsi) = \bm B \bm\Psi \bm B^\intercal \odot \bm I\) is the random diagonal matrix of \(\bPsi\) determined by the basis \(\mathbf{B}\).

	Then minimizing the expected squared Frobenius norm of the residual (off-diagonal component) of \(\bPsi\):
	\[
		\underset{\mathbf{B}}{\minimize} \ \mathbb{E}\left[\left\| \bPsi - \mathcal{P}_{\mathbf{B}}(\bPsi) \right\|_{\F}^2\right],
	\]
	is equivalent to maximizing the expected squared Frobenius norm of the projection:
	\[
		\underset{\mathbf{B}}{\maximize} \ \mathbb{E}\left[\left\| \mathcal{P}_{\mathbf{B}}(\bPsi)  \right\|_{\F}^2\right],
	\]
	which can also be expressed as maximizing the expected Froenius norm of the diagonal component in the basis \(\mathbf{B}\):
	\[
		\underset{\mathbf{B}}{\maximize} \ \mathbb{E}\left[\left\| \bm{\Lambda}_{\mathbf{B}}(\bPsi) \right\|_{\F}^2\right].
	\]
\end{proposition}

\begin{proof}[Proof of Proposition~\ref{sec4:prop_equiv}]
	To obtain the common principal components $\mathbf B$, the goal is to find an orthonormal matrix that minimizes \cref{sec3:cpcestimation}:
	\begin{align}
		\E\left[ \left\Vert \bm\Psi - \mathcal P_{\mathbf B}(\bm\Psi) \right\Vert_{\F}^2  \right] \notag
		 & = \E \left[ \left\Vert \mathbf B \left(\mathbf B^\intercal \bm\Psi \mathbf B \right) \mathbf B^\intercal - \mathbf B \left(\bm\Lambda_{\mathbf B}(\bm\Psi)\right) \mathbf B^\intercal \right\Vert_{\F}^2 \right]  \notag                                                                                                                            \\
		 & = \E \left[ \left\Vert \mathbf B^\intercal \bm\Psi \mathbf B - \left(\mathbf B^\intercal \bm\Psi \mathbf B \right) \odot  \mathbf I  \right\Vert_{\F}^2 \right] \qquad (\text{by \Cref{sec3:thm1}})  \notag                                                                                                                                         \\
		 & = \E \left[ \tr\left\{ \left( \bm\Psi^\intercal \bm\Psi \right) - 2\left(\mathbf B^\intercal \bm\Psi \mathbf B \right)^\intercal\left( \left(\mathbf B^\intercal \bm\Psi \mathbf B \right) \odot   \mathbf I   \right) + \left( \left( \mathbf B^\intercal \bm\Psi \mathbf B \right) \odot  \mathbf I  \right)^2 \right\} \right] \label{sec4:eq15} \\
		 & =  \E \Bigl[ \tr \left( \bPsi^\intercal \bPsi \right) \Bigr]  - \E \left[\tr \left(\left( \Bigl( \mathbf B^\intercal \bm\Psi \mathbf B \right) \odot  \mathbf I  \Bigr)^2\right)    \right]   \notag
	\end{align}
	where in \cref{sec4:eq15}, denote \( (\bm B^\intercal \bPsi \bm B) \odot \mathbf I \) as \(\mathbf D = \text{diag}(d_1,\ldots, d_p)\), here \(d_j = \left( \bm B^\intercal \bm\Psi \bm B \right)_{jj} = \bm b_j^\intercal \bm\Psi \bm b_j\), and \(\bm b_j\) is the \(j\)-th column of \(\mathbf B\) for \(j = 1,\ldots, p\).
	The second term in this last expression simplifies to
	$$\sum_{j=1}^p{d_j^2}.$$
	Since, \(\E \Bigl[ \tr \left( \bPsi^\intercal \bPsi \right) \Bigr]\) is independent with \(\bm B\), to minimize $\E\left[ \left\Vert \bm\Psi - \mathcal P_{\mathbf B}(\bm\Psi) \right\Vert_{\F}^2  \right]$ over $\mathbf B$ is equivalent to maximize
	\begin{align}
		\E \Big[\Vert  \big( \mathcal P_{\mathbf B}(\bPsi)\big) \Vert_F^2 \Big] = \E \Big[ \tr \Big( \left( \left( \mathbf B^\intercal \bm\Psi \mathbf B \right) \odot  \mathbf I  \right)^2  \Big) \Big]\label{sec4:eq17} = \mathbb E \Big[\Vert \bm\Lambda_{\bm B}(\bPsi) \Vert_F^2\Big] = \sum_{j=1}^p \mathbb{E} \left[ \left( \bm{b}_j^\intercal \bm{\Psi} \bm{b}_j \right)^2 \right]
	\end{align}
	over \(\bm B\), where \(\bm B = [\bm b_1, \ldots, \bm b_p] \in \mathcal V_p(\mathbb R^p)\) is an orthogonal basis matrix.
\end{proof}

Building upon \cref{sec4:prop_equiv}, we provide a least-squares solution of CPC via minimizing the expected Frobenius norm of the projection index \(\bm \Lambda_{\bm B}(\bm \Psi)\).

\begin{theorem}[Common Principal Components via Projection Index Maximization] \label{prop:cpc_fg_algo}
	Let \( \bm{\Psi} \in \mathbb{S_+}^{p} \) be a random symmetric positive semi-definite matrix, and {let} \( \bm{B} = [\bm{b}_1, \bm{b}_2, \ldots, \bm{b}_p] \in \mathcal V_p(\mathbb R^p) \).
	If the columns of \( \bm{B} \) are such that
	\begin{equation}\label{eq:opt_cond_cpc}
		\text{maximize} \quad
		\sum_{j=1}^p \mathbb{E} \left[ \left( \bm{b}_j^\intercal \bm{\Psi} \bm{b}_j \right)^2 \right],
	\end{equation}
	then the following optimality condition holds for all \( j \ne h \):
	\begin{align}
		\bm{b}_j^\intercal \, \mathbb{E} \left[ \left( \bm{b}_j^\intercal \bm{\Psi} \bm{b}_j - \bm{b}_h^\intercal \bm{\Psi} \bm{b}_h \right) \bm{\Psi} \right] \bm{b}_h = 0.
	\end{align}
\end{theorem}
This optimality condition is equivalent to the maximum likelihood estimates shown in \citet{flury1984common}, and can be solved by the algorithm proposed in \citet{flury1986algorithm}.
\begin{proof}
	To prove Theorem~\ref{prop:cpc_fg_algo} and obtain a solution for \( \bm{B} \),
	Using the properties of an orthogonal matrix with orthonormal columns,
	we can formulate the Lagrangian function with respect to each column space of \( \bm{B} \) as follows (\( \bm{B} = [\bm{b}_1, \bm{b}_2, \ldots, \bm{b}_p] \)):
	\begin{align}
		\mathcal{L}(\bm{b}_1,\ldots,\bm{b}_p \mid \bm{\Psi})
		= \sum_{j=1}^p \mathbb{E} \left[ (\bm{b}_j^\intercal \bm{\Psi} \bm{b}_j)^2 \right]
		- 2 \sum_{j=1}^p \sum_{h>j}^p \gamma_{hj} \bm{b}_h^\intercal \bm{b}_j
		- \sum_{h=1}^p \gamma_h \bm{b}_h^\intercal \bm{b}_h,
	\end{align}
	where \( \bm{b}_j \) and \( \bm{b}_h \) are the \( j^{\text{th}} \) and \( h^{\text{th}} \) columns of \( \bm{B} \). Taking the partial derivatives with respect to \( \bm{b}_\ell \) and setting {them} to zero, we obtain:
	\begin{align} \label{sec4:eq20}
		\frac{\partial \mathcal{L}}{\partial \bm{b}_\ell}
		= 2 \mathbb{E} \left[ (\bm{b}_\ell^\intercal \bm{\Psi} \bm{b}_\ell) \bm{\Psi} \bm{b}_\ell \right]
		- 2 \sum_{\substack{h=1 \\ h\neq\ell}}^p \gamma_{h\ell} \bm{b}_h
		- 2 \gamma_\ell \bm{b}_\ell = 0.
	\end{align}
	Multiplying Equation~\eqref{sec4:eq20} on the left by \( \frac{1}{2} \bm{b}_\ell^\intercal \),
	\begin{align}\label{eq:multiplied_by_beta_l}
		\frac{\partial \mathcal{L}}{\partial \bm{b}_\ell}
		= \mathbb{E} \left[ \bm{b}_\ell^\intercal (\bm{b}_\ell^\intercal \bm{\Psi} \bm{b}_\ell) \bm{\Psi} \bm{b}_\ell \right]
		-  \sum_{\substack{h=1 \\ h\neq\ell}}^p \gamma_{h\ell} \bm{b}_\ell^\intercal \bm{b}_h
		-  \gamma_\ell \bm{b}_\ell^\intercal \bm{b}_\ell = 0,
	\end{align}
	since \(\bm{b}_\ell^\intercal \bm{\Psi} \bm{b}_\ell\) is a scalar, the first term in \cref{eq:multiplied_by_beta_l} \(\bm{b}_\ell^\intercal (\bm{b}_\ell^\intercal \bm{\Psi} \bm{b}_\ell) \bm{\Psi} \bm{b}_\ell  = (\bm{b}_\ell^\intercal \bm{\Psi} \bm{b}_\ell)^2 \), the second term and third term in \cref{eq:multiplied_by_beta_l} becomes \(0\), and \(\gamma_\ell\) respectively. Thus, it is obtained that:
	\[
		\mathbb{E} \left[ (\bm{b}_\ell^\intercal \bm{\Psi} \bm{b}_\ell)^2 \right]
		- \sum_{\substack{h=1\\ h\neq\ell}}^p \gamma_{h\ell} \bm{b}_\ell^\intercal \bm{b}_h
		- \gamma_\ell \bm{b}_\ell^\intercal \bm{b}_\ell = 0,
	\]
	implying that
	\[
		\gamma_\ell = \mathbb{E} \left[ (\bm{b}_\ell^\intercal \bm{\Psi} \bm{b}_\ell)^2 \right].
	\]
	Now substitute \( \gamma_\ell \) back into Equation~\eqref{sec4:eq20} and multiply it from the left by \( \bm{b}_m^\intercal \), which implies:
	\begin{align} \label{sec4:eq22}
		\mathbb{E} \left[ (\bm{b}_\ell^\intercal \bm{\Psi} \bm{b}_\ell) \cdot (\bm{b}_m^\intercal \bm{\Psi} \bm{b}_\ell) \right]
		- \sum_{\substack{h=1 \\ h\neq\ell}}^p \gamma_{h\ell} \bm{b}_m^\intercal \bm{b}_h
		- \mathbb{E} \left[ (\bm{b}_\ell^\intercal \bm{\Psi} \bm{b}_\ell)^2 \right] \bm{b}_m^\intercal \bm{b}_\ell = 0.
	\end{align}
	For \( m \neq \ell \), the second term in \cref{sec4:eq22} becomes \( \sum_{\substack{h=1\\ h\neq\ell}}^p \gamma_{h\ell} \bm{b}_m^\intercal \bm{b}_h = \gamma_{m\ell} \), when \(h = m\). Therefore, \cref{sec4:eq22} implies that
	\[
		\gamma_{m\ell} = \mathbb{E} \left[ (\bm{b}_\ell^\intercal \bm{\Psi} \bm{b}_\ell)(\bm{b}_m^\intercal \bm{\Psi} \bm{b}_\ell) \right]
		= \mathbb{E} \left[ (\bm{b}_\ell^\intercal \bm{\Psi} \bm{b}_\ell)(\bm{b}_\ell^\intercal \bm{\Psi} \bm{b}_m) \right],
	\]
	since \( \bm{b}_\ell^\intercal \bm{\Psi} \bm{b}_m = \bm{b}_m^\intercal \bm{\Psi} \bm{b}_\ell \). Now, by setting \( \gamma_{m\ell} = \gamma_{\ell m} \),
	\[
		\mathbb{E} \left[ (\bm{b}_\ell^\intercal \bm{\Psi} \bm{b}_\ell)(\bm{b}_m^\intercal \bm{\Psi} \bm{b}_\ell) \right] = \mathbb{E} \left[ (\bm{b}_m^\intercal \bm{\Psi} \bm{b}_m)(\bm{b}_\ell^\intercal \bm{\Psi} \bm{b}_m) \right] = \mathbb{E} \left[ (\bm{b}_m^\intercal \bm{\Psi} \bm{b}_m)(\bm{b}_m^\intercal \bm{\Psi} \bm{b}_\ell) \right]
	\]
	we obtain:
	\begin{align} \label{sec4:eq22_final}
		\bm{b}_\ell^\intercal \, \mathbb{E} \left[ \left( \bm{b}_\ell^\intercal \bm{\Psi} \bm{b}_\ell - \bm{b}_m^\intercal \bm{\Psi} \bm{b}_m \right) \bm{\Psi} \right] \bm{b}_m = 0.
	\end{align}
	This holds for all \( \ell, m = 1, \ldots, p \), with \( m \neq \ell \).
\end{proof}

\subsection{Principal SPSD Tensors and Self-Consistent SPSD Tensors} \label{sec2.3}

Let $\mathbf{B}_1, \ldots, \mathbf{B}_K \in \mathcal{V}_p(\mathbb{R}^p)$ denote $K$ distinct {orthogonal} matrices. The minimal distance of a given $\bm{\psi} \in \mathcal{S}_+^p$ to the set $\left\{\mathcal{P}_{\mathbf{B}_k}(\bm{\psi}) \right\}_{k=1}^K$ is defined as
\begin{align*}
	\mathbf{d}^2(\bm{\psi} | \mathcal{P}_{\mathbf{B}_1}(\bm{\psi}), \ldots, \mathcal{P}_{\mathbf{B}_K}(\bm{\psi})) = \underset{k \in \{1,\dots, K\}}{\min} \left\Vert \bm{\psi} - \mathcal{P}_{\mathbf{B}_k}(\bm{\psi}) \right\Vert_{F}^2.
\end{align*}

\begin{definition}\label{def:kprincipal_tensors}
	We call $\mathcal{P}_{\mathbf{B}_1}(\bm{\Psi}), \ldots, \mathcal{P}_{\mathbf{B}_K}(\bm{\Psi})$ \(K\) \textit{principal positive semi-definite tensors} of the random matrix \(\bm{\Psi}\) if
	\[
		\mathbb{E} \left[ \mathbf{d}^2\left( \bm{\Psi} \,\middle\vert\, \mathcal{P}_{\mathbf{B}_1}(\bm{\Psi}), \ldots, \mathcal{P}_{\mathbf{B}_K}(\bm{\Psi}) \right) \right]
		=
		\min_{\mathbf{B}_1',\ldots,\mathbf{B}_K' \in \mathcal{V}_p(\mathbb{R}^p)}
		\mathbb{E} \left[ \mathbf{d}^2\left( \bm{\Psi} \,\middle\vert\, \mathcal{P}_{\mathbf{B}_1'}(\bm{\Psi}), \ldots, \mathcal{P}_{\mathbf{B}_K'}(\bm{\Psi}) \right) \right],
	\]
	for any collection $K$ orthogonal matrices $\{\mathbf{B}_1', \dots, \mathbf{B}_K'$\}.
\end{definition}
\begin{definition}\label{sec3:def1}
	Let \( \mathcal{B} \in \{ \bm{B}_1, \ldots, \bm{B}_K \} \), where \(\bm B_k \in \mathcal V_p(\mathbb R^p)\), we {call} \( \mathcal{P}_{\mathcal{B}}(\bm{\Psi}) \)  {a set of} \emph{K self-consistent SPSD tensors} of \( \bm{\Psi} \) if it satisfies the self-consistency condition:
	\begin{align}
		\mathbb{E}\left[\bm{\Psi} \,\middle|\, \bm{\Lambda}_{\mathcal{B}}(\bm{\Psi})\right]
		= \mathcal{P}_{\mathcal{B}}(\bm{\Psi}) \qquad \text{a.s.}.
	\end{align}
\end{definition}
\Cref{sec3:def1} is analogous to the notion of self-consistency defined in \cref{eq:principal_points}.
Conditionally, we call \(\mathcal P_{\bm B}(\bm \Psi)\) the \emph{self-consistent tensor} of \(\bm \Psi\) if
\begin{align}
	\mathbb{E}\left[\bm{\Psi} \,\middle|\, \mathcal{B} = \mathbf{B}\right]
	= \mathbf{B} \bm{\Lambda}_{\mathbf{B}}(\bm{\Psi}) \mathbf{B}^\intercal \quad \text{a.s.}.
\end{align}
Finally, we have a pointwise self-consistency condition:
\begin{align}\label{eq:pointwise_sc}
	\mathbb{E}\left[\bm{\Psi} \,\middle|\, \mathcal{B} = \mathbf{B},\; \bm{\Lambda}_{\mathbf{B}}(\bm{\Psi}) = \bm{\Lambda} \right]
	= \mathbf{B} \bm{\Lambda} \mathbf{B}^\intercal \quad \text{a.s.}.
\end{align}
The self-consistency of the principal SPSD tensor presented in \cref{eq:pointwise_sc} is established under the Wishart distribution in \Cref{thm:self-consistency_spsd_tensors}.

\subsection{Domain of Attraction to an Principal SPSD Tensors} \label{sec2.4}
An equally important concept is that of the \textit{domain of attraction}, which formalizes how individual observations are associated with these principal structures. For each principal SPSD tensor \( \mathcal{P}_{\bm{B}_k}(\bm{\Psi}) \), parameterized by a basis matrix \( \mathbf{B}_k \in \mathcal{V}_p(\mathbb{R}^p) \), the domain of attraction defines a region in the space of SPSD matrices where the observed matrices exhibit eigenstructures most aligned with \( \mathbf{B}_k \). This partitioning forms the foundation for assigning observations to their closest CPC subspace and serves as a key building block in clustering and other subspace-based learning tasks.
\begin{definition}\label{sec3:domain}
	Given a set of distinct $\B_\k \in\mathcal V_p(\R^p), \; k=1,\dots, \K$,
	we define $\mathcal D_\k$ the \textit{domain of attraction} for each
	$\mathbf B_\k$ as:
	\begin{align}\label{sec3: assignments}
		\mathcal D_k := \Big\{ & \bpsi \in \S_+^p,  :
		\left\Vert \bpsi - \mathcal P_{\mathbf B_k}(\bpsi) \right\Vert_{\F}^2 = \inf_{h} \{\left\Vert \bpsi - \mathcal P_{\mathbf B_h}(\bpsi)
		\right\Vert_{\F}^2 : h=1,...,K \} \Big\}{.}
	\end{align}
\end{definition}
An observation is assigned to a specific set $\mathcal D_\k$ if its projection, parameterized by $\B_\k$, is closest to the observation compared to the projections using other $\B_h$ for $h = 1,\ldots, K$ from the set.
Throughout, it is assumed that the underlying distributions of SPSD matrices under consideration are continuous, in which case the boundaries between different domains of attraction have probability zero and can therefore be ignored.

\section{$K$-Tensors Clustering Algorithm for SPSD Matrices}
\label{sec3}


\subsection{$K$-Tensors Algorithm}

Drawing on the concepts introduced in the preceding sections, we propose a novel iterative clustering algorithm, called $K$-Tensors, to cluster a collection $\{\bpsi_1, \dots, \bpsi_n\}$ of SPSD matrices. This algorithm utilizes the principal SPSD tensors to approximate the distribution of each cluster, resulting in cluster assignments that reflect differences in eigenstructure. The $K$-Tensors algorithm comprises two key steps: (1) identifying the common eigenvectors within each cluster of SPSD matrices using the F-G algorithm \citep{flury1984common}, and (2) assigning all matrices to their closest cluster based on the shape information defined in \Cref{sec3:domain}. The algorithm iterates between these steps until convergence is achieved.

\begin{algorithm}[H]
	\caption{$K$-Tensors with FG Algorithm for CPC Estimation}\label{algo:ktensors}
	\label{algo:ktensor_fg}
	\begin{algorithmic}[1]
		\STATE \textbf{Input:} SPSD matrices \( \{\bm{\psi}_1, \dots, \bm{\psi}_n\} \subset \mathbb{S}_+^p \), number of clusters \( K \), tolerance \( \texttt{tol} \)
		\STATE \textbf{Initialize:} Randomly partition \( \{\bm{\psi}_i\}_i^n \) into non-overlapping \( \{\mathcal{D}_k\}_{k=1}^K \)
		\REPEAT
		\FOR{each cluster \( k = 1, \dots, K \)}
		\STATE Initialize \( \mathbf{B}_k \leftarrow \mathbf{I}_p \)
		\REPEAT
		\STATE Store previous \( \mathbf{B}_k^{\text{old}} \leftarrow \mathbf{B}_k \)
		\FOR{each pair \( (l, j) \) with \( 1 \leq l < j \leq p \)}
		\STATE Compute matrices \( \bm T_i^{(l,j)} =
		\begin{bmatrix}
			\mathbf{b}_{k,l}^\intercal \bm{\psi}_i \mathbf{b}_{k,l} & \mathbf{b}_{k,l}^\intercal \bm{\psi}_i \mathbf{b}_{k,j} \\
			\mathbf{b}_{k,j}^\intercal \bm{\psi}_i \mathbf{b}_{k,l} & \mathbf{b}_{k,j}^\intercal \bm{\psi}_i \mathbf{b}_{k,j}
		\end{bmatrix} \), where \(\bm b_{k,l}, \bm b_{k,j}\) are the $l$-th and $j$-th column of \(\bm B_k\).
		\STATE Initialize \( \mathbf{Q} \leftarrow \mathbf{I}_2 \)
		\REPEAT
		\STATE Store \( \mathbf{Q}_{\text{old}} \leftarrow \mathbf{Q} \)
		\STATE Compute:
		\(
		\mathbf{U} = \sum_{i} \cdot \frac{d_{i1} - d_{i2}}{d_{i1} d_{i2}} \cdot \bm T_i^{(l,j)}
		\)
		where \( d_{i1} = \mathbf{q}_1^\intercal \bm T_i^{(l,j)} \mathbf{q}_1, d_{i2} = \mathbf{q}_2^\intercal \bm T_i^{(l,j)} \mathbf{q}_2 \), \(\bm q_1, \bm q_2\) are the first and second column of matrix \(\bm Q\)
		\STATE Compute eigen-decomposition \( \mathbf{U} = \mathbf{V} \bm{\Lambda} \mathbf{V}^\intercal \)
		\STATE Update \( \mathbf{Q} \leftarrow \mathbf{V} \)
		\UNTIL{ \( \| \mathbf{Q} - \mathbf{Q}_{\text{old}} \| < \texttt{tol} \) }
		\STATE Update columns \( l \) and \( j \): \( \mathbf{B}_k[:, \{l,j\}] \leftarrow \mathbf{B}_k[:, \{l,j\}] \cdot \mathbf{Q} \)
		\ENDFOR
		\UNTIL{ \( \| \mathbf{B}_k - \mathbf{B}_k^{\text{old}} \| < \texttt{tol} \) }
		\ENDFOR
		\STATE \textbf{Update Assignments:}
		\FOR{each observation \( \bm{\psi}_i \)}
		\STATE Assign to cluster:
		\[
			\mathcal{D}_k \leftarrow \arg\min_{k} \left\| \bm{\psi}_i - \mathcal{P}_{\mathbf{B}_k}(\bm{\psi}_i) \right\|_F^2
		\]
		\ENDFOR
		\UNTIL{ assignments \( \{ \mathcal{D}_k \}_{k=1}^K \) do not change }
		\STATE \textbf{Output:} Clusters \( \{ \mathcal{D}_k \}_{k=1}^K \) and CPC bases \( \{ \mathbf{B}_k \}_{k=1}^K \)
	\end{algorithmic}
\end{algorithm}

\begin{theorem}[CPC for Matrices with Constant Directional Kurtosis]\label{prop:cpc_constant_kurtosis}
	Let \( \bm{\Psi} \in \mathbb{S}_+^p \) be a random symmetric positive semi-definite matrix with \( \mathbb{E}[\bm{\Psi}] = \bm{\Sigma} \), and let \( \bm{B} = [\bm{b}_1, \ldots, \bm{b}_p] \in \mathcal{V}_p(\mathbb{R}^p) \). Suppose that for all unit vectors \( \bm{b} \in \mathbb{R}^p \), the kurtosis of the projected quadratic form is constant, i.e.,
	\[
		\frac{\mathbb{E}[(\bm{b}^\intercal \bm{\Psi} \bm{b})^2]}{(\bm{b}^\intercal \mathbb{E}[\bm{\Psi}] \bm{b})^2} = c, \quad \text{for some constant } c > 0.
	\]
	Then the columns of \( \bm{B} \) that maximize \cref{eq:opt_cond_cpc}:
	\[
		\sum_{j=1}^p \mathbb{E}[(\bm{b}_j^\intercal \bm{\Psi} \bm{b}_j)^2]
		\quad \text{subject to} \quad \bm{B}^\intercal \bm{B} = \bm{I}_p
	\]
	are given by the eigenvectors of \( \mathbb{E}[\bm{\Psi}] \). That is, the CPC directions correspond to the eigenvectors of \( \bm{\Sigma} \).
\end{theorem}
\begin{proof}\label{proof:cpc_constant_kurtosis}
	Let \( \bm{b} \in \mathbb{R}^p \) be a unit vector, i.e., \( \| \bm{b} \| = 1 \). Consider the scalar random variable \( Z = \bm{b}^\intercal \bm{\Psi} \bm{b} \), which is a quadratic form in \( \bm{\Psi} \). Its expectation is:
	\[
		\mathbb{E}[Z] = \mathbb{E}[\bm{b}^\intercal \bm{\Psi} \bm{b}] = \bm{b}^\intercal \mathbb{E}[\bm{\Psi}] \bm{b} = \bm{b}^\intercal \bm{\Sigma} \bm{b}.
	\]
	Let us expand the second moment \( \mathbb{E}[Z^2] \) using the variance decomposition:
	\[
		\mathbb{E}[Z^2] = \text{Var}(Z) + \left( \mathbb{E}[Z] \right)^2 = \text{Var}(\bm{b}^\intercal \bm{\Psi} \bm{b}) + \left( \bm{b}^\intercal \bm{\Sigma} \bm{b} \right)^2.
	\]
	Under the assumption of constant directional kurtosis, we have:
	\[
		\frac{\mathbb{E}[Z^2]}{(\bm{b}^\intercal \bm{\Sigma} \bm{b})^2} = c \quad \Rightarrow \quad \mathbb{E}[Z^2] = c \cdot (\bm{b}^\intercal \bm{\Sigma} \bm{b})^2.
	\]
	Substituting into the variance decomposition:
	\[
		\text{Var}(\bm{b}^\intercal \bm{\Psi} \bm{b}) = \mathbb{E}[Z^2] - (\mathbb{E}[Z])^2 = \left( c - 1 \right) \cdot (\bm{b}^\intercal \bm{\Sigma} \bm{b})^2.
	\]
	Now, for a set of orthonormal directions \( \bm{b}_1, \ldots, \bm{b}_p \in \mathbb{R}^p \), define the objective:
	\[
		\sum_{j=1}^p \mathbb{E} \left[ \left( \bm{b}_j^\intercal \bm{\Psi} \bm{b}_j \right)^2 \right].
	\]
	Applying the constant kurtosis assumption, each term becomes:
	\[
		\mathbb{E} \left[ \left( \bm{b}_j^\intercal \bm{\Psi} \bm{b}_j \right)^2 \right] = c \cdot \left( \bm{b}_j^\intercal \bm{\Sigma} \bm{b}_j \right)^2.
	\]
	Thus, the total objective becomes:
	\[
		\sum_{j=1}^p \mathbb{E} \left[ \left( \bm{b}_j^\intercal \bm{\Psi} \bm{b}_j \right)^2 \right]
		= c \sum_{j=1}^p \left( \bm{b}_j^\intercal \bm{\Sigma} \bm{b}_j \right)^2.
	\]
	Since \( c > 0 \) is a constant, maximizing the objective is equivalent to solving
	\begin{align}\label{eq:cpc_sigma_loss_copy}
		\max_{\bm{B} \in \mathcal{V}_p(\mathbb{R}^p)} \sum_{j=1}^p \left( \bm{b}_j^\intercal \bm{\Sigma} \bm{b}_j \right)^2,
	\end{align}
	for \(\bm{B} = [\bm{b}_1, \ldots, \bm{b}_p] \in \mathcal{V}_p(\mathbb{R}^p)\) be an orthonormal matrix. Then the objective in \eqref{eq:cpc_sigma_loss_copy} can be written as
	\begin{align}
		\sum_{j=1}^p \left( \bm{b}_j^\intercal \bm{\Sigma} \bm{b}_j \right)^2
		= \left\Vert \text{diag}\left( \bm{B}^\intercal \bm{\Sigma} \bm{B} \right) \right\Vert_F^2
		\le \left\Vert \bm{B}^\intercal \bm{\Sigma} \bm{B} \right\Vert_F^2
		= \left\Vert \bm{\Sigma} \right\Vert_F^2,
	\end{align}
	with equality if and only if \( \bm{B}^\intercal \bm{\Sigma} \bm{B} \) is diagonal. This condition holds precisely when each \( \bm{b}_j \) is an eigenvector of \( \bm{\Sigma} \) with eigenvalue \( \lambda_j \), for \( j = 1, \ldots, p \). Therefore, the objective is maximized when \( \bm{b}_1, \ldots, \bm{b}_p \) are the eigenvectors of \( \bm{\Sigma} \). In particular, the CPC directions \( \bm{B} = [\bm{b}_1, \ldots, \bm{b}_p] \) of \( \bm{\Psi} \) coincide with the eigenvectors of the matrix \( \bm{\Sigma} \), as desired.

\end{proof}

\begin{corollary}[CPC for Wishart Distribution]\label{prop:wishart}
	Assume \( \bm{\Psi} \sim \mathcal{W}_p(n, \bm{\Sigma}) \), and let \( \bm{B} \) denote the common principal components (CPC) of the random matrix \( \bm{\Psi} \). Then, the CPC corresponds to the eigenvectors of the scale matrix \( \bm{\Sigma} \).
\end{corollary}
\begin{proof}
	Assume \( \bm{\Psi} \sim \mathcal{W}_p(n, \bm{\Sigma}) \). By the definition of the Wishart distribution, there exist i.i.d. random vectors \( \bm{x}_1, \ldots, \bm{x}_n \sim \mathcal{N}_p(\bm{0}, \bm{\Sigma}) \) such that
	\[
		\bm{\Psi} = \sum_{i=1}^n \bm{x}_i \bm{x}_i^\intercal.
	\]
	Let \( \bm{b}_j \in \mathbb{R}^p \) be any unit vector. Consider the scalar random variable
	\[
		\bm{b}_j^\intercal \bm{\Psi} \bm{b}_j = \sum_{i=1}^n \bm{b}_j^\intercal \bm{x}_i \bm{x}_i^\intercal \bm{b}_j = \sum_{i=1}^n (\bm{b}_j^\intercal \bm{x}_i)^2.
	\]
	Note that for each \( i \), the scalar \( \bm{b}_j^\intercal \bm{x}_i \sim \mathcal{N}(0, \bm{b}_j^\intercal \bm{\Sigma} \bm{b}_j) \), i.e., it is a Gaussian with zero mean and variance \( \bm{b}_j^\intercal \bm{\Sigma} \bm{b}_j \). Therefore,
	\[
		\frac{\bm{b}_j^\intercal \bm{x}_i}{\sqrt{\bm{b}_j^\intercal \bm{\Sigma} \bm{b}_j}} \sim \mathcal{N}(0,1),
	\]
	and hence
	\[
		\frac{\bm{b}_j^\intercal \bm{\Psi} \bm{b}_j}{\bm{b}_j^\intercal \bm{\Sigma} \bm{b}_j}
		= \sum_{i=1}^n \left( \frac{\bm{b}_j^\intercal \bm{x}_i}{\sqrt{\bm{b}_j^\intercal \bm{\Sigma} \bm{b}_j}} \right)^2
		\sim \chi^2_n.
	\]
	This implies that
	\[
		\bm{b}_j^\intercal \bm{\Psi} \bm{b}_j \sim (\bm{b}_j^\intercal \bm{\Sigma} \bm{b}_j) \cdot \chi^2_n.
	\]
	Now consider the objective:
	\[
		\sum_{j=1}^p \mathbb{E} \left[ \left( \bm{b}_j^\intercal \bm{\Psi} \bm{b}_j \right)^2 \right].
	\]
	Using the result that the second moment of a chi-squared random variable with \( n \) degrees of freedom is
	\[
		\mathbb{E}\left[ (\chi^2_n)^2 \right] = n(n + 2),
	\]
	it follows that
	\[
		\mathbb{E} \left[ \left( \bm{b}_j^\intercal \bm{\Psi} \bm{b}_j \right)^2 \right]
		= (\bm{b}_j^\intercal \bm{\Sigma} \bm{b}_j)^2 \cdot \mathbb{E}[(\chi^2_n)^2]
		= (\bm{b}_j^\intercal \bm{\Sigma} \bm{b}_j)^2 \cdot n(n+2).
	\]
	Therefore,
	\begin{align}\label{eq:cpc_sigma_loss}
		\sum_{j=1}^p \mathbb{E} \left[ \left( \bm{b}_j^\intercal \bm{\Psi} \bm{b}_j \right)^2 \right]
		= n(n+2) \sum_{j=1}^p (\bm{b}_j^\intercal \bm{\Sigma} \bm{b}_j)^2.
	\end{align}
	The rest of the proof follows from \Cref{prop:cpc_constant_kurtosis}, showing that maximizing \cref{eq:cpc_sigma_loss} under the orthonormality constraint on \( \bm{b}_j \) for \( j = 1, \ldots, p \) results in the eigenvectors of \( \bm{\Sigma} \).
\end{proof}

\begin{remark}\label{sec4:remark}
	By invoking \Cref{prop:cpc_constant_kurtosis} and \Cref{prop:wishart}, we can simplify the process of obtaining a set of K orthogonal matrices associated with each cluster under the assumption of constant kurtosis or Wishart distribution. This process is a crucial part of our clustering algorithm, where each cluster is characterized by its own distinct orthogonal matrix, effectively implementing a common principal component analysis approach. The result of this simplification is a more streamlined algorithm, which we refer to as Fast $K$-Tensors, detailed in \Cref{fastKTensor}. This version of the algorithm incorporates an efficient method for estimating common principal components, thereby enhancing the overall performance and utility of the clustering procedure.
\end{remark}

\begin{algorithm}[htbp!]\label{sec4:algo2}
	Initialize clusters $\{\mathcal D_\k\}_{\k = 1}^K$ ({e.g., }by randomly partitioning the data $\{\bpsi_1, \dots, \bpsi_n\}$ )\\
	\textbf{Iterate} until the cluster assignment $\{\mathcal D_\k\}_{\k = 1}^K$ doesn't change:
	\begin{enumerate}
		\item Estimate common principal components for each cluster and update $\{\mathbf B_{\k}\}_{k=1}^K$ by computing the eigenvectors of the sample covariance matrix for each cluster.
		\item Obtain new cluster assignments for each observation by assigning each matrix $\bpsi_i$ to the cluster whose projection $\mathcal P_{\mathbf B_k}(\bpsi_i)$ minimizes the Frobenius norm $\Vert \bpsi_i - \mathcal P_{\mathbf B_k}(\bpsi_i) \Vert_F^2$, and update $\{\mathcal D_\k\}_{\k = 1}^K$ accordingly (\cref{sec3: assignments}).
	\end{enumerate}
	\caption{Fast $K$-Tensors: Clustering Positive Semi-Definite Matrices}\label{fastKTensor}
\end{algorithm}

\section{{Self-Consistency} Property}
This section proves that, under the Wishart distribution assumption, the principal SPSD tensors satisfy the self-consistency property \ref{sec3:def1}. Consequently, the Fast $K$-Tensors algorithm constitutes a self-consistent clustering {algorithm}.

\begin{theorem}[Principal and Self-Consistent SPSD Tensors]\label{thm:self-consistency_spsd_tensors}
	Assume \( \bm{\Psi} \sim \mathcal{W}_p(v, \bm{\Sigma}) \) is a Wishart-distributed random matrix. Let \( \bm{B} \) denote the matrix of eigenvectors of \( \bm{\Sigma} \). Then:
	\begin{align}
		\mathbb{E}[\bm{\Psi} \mid \mathcal{B} = \bm{B},\; \bm{\Lambda}_{\bm{B}}(\bm{\Psi}) = \bm{\Lambda}] = \bm{B} \bm{\Lambda} \bm{B}^\intercal.
	\end{align}
\end{theorem}
\begin{proof}
	Let \(\bm\Psi \sim \mathcal{W}_p(v, \bm\Sigma)\), and \(\bm B\) be the matrix of eigenvectors for \(\bm \Sigma\). By properties of the Wishart distribution, the following results hold:

	\textbf{Statement 1:} The matrix \(\bm{B}^\intercal \bm{\Psi} \bm{B}\) is also Wishart distributed:
	\[
		\bm{B}^\intercal \bm{\Psi} \bm{B} \sim \mathcal{W}_p(v, \bm{D}),
	\]
	where \(\bm{D}\) is the diagonal matrix of eigenvalues of \(\bm{\Sigma}\).

	\textbf{Statement 2:} Let \(\Psi_{jj}\) and \(\Psi_{jh}\) denote the \(j\)-th diagonal entry and the \((j,h)\)-th off-diagonal entry of \(\bm{\Psi}\), respectively. Then, the \(2 \times 2\) principal submatrix
	\[
		\tilde{\bm{\Psi}}_{jh} =
		\begin{bmatrix}
			\Psi_{jj} & \Psi_{jh} \\
			\Psi_{jh} & \Psi_{hh}
		\end{bmatrix}
		\sim \mathcal{W}_2(v, \tilde{\bm{\Sigma}}_{jh}),
		\quad \text{where } \tilde{\bm{\Sigma}}_{jh} =
		\begin{bmatrix}
			\Sigma_{jj} & \Sigma_{jh} \\
			\Sigma_{jh} & \Sigma_{hh}
		\end{bmatrix},
	\]
	also follows a Wishart distribution.

	The proof begins by considering the special case where \(\bm \Phi \sim \mathcal{W}_p(v, \bm{D})\) with \(\bm{D}\) diagonal. Then the principal submatrix
	\[
		\tilde{\bm{\Phi}}_{jh} \sim \mathcal{W}_2(v, \tilde{\bm{D}}_{jh}),
		\quad \text{with } \ \tilde{\bm{D}}_{jh} = \begin{bmatrix} d_j & 0 \\ 0 & d_h \end{bmatrix},
	\]
	where \(d_j\) and \(d_h\) are the \(j\)-th and \(h\)-th diagonal entries of \(\bm{D}\).

	The joint density of \(\Phi_{jj}, \Phi_{jh}, \Phi_{hh}\) follows directly from the distribution of the \(2 \times 2\) principal submatrix \(\tilde{\bm{\Phi}}_{jh}\), as stated in Statement 2, and is given by:
	\begin{align*}
		f_{\Phi_{jj}, \Phi_{jh}, \Phi_{hh}}(\phi_{jj}, \phi_{jh}, \phi_{hh})
		= \frac{
			(\phi_{jj} \phi_{hh} - \phi_{jh}^2)^{\frac{v - 3}{2}}
		}{
			2^v (d_j d_h)^{v/2} \Gamma_2\left(\frac{v}{2}\right)
		}
		\exp\left(
		-\frac{1}{2} \left(\frac{\phi_{jj}}{d_j} + \frac{\phi_{hh}}{d_h}\right)
		\right).
	\end{align*}
	To obtain the joint density of \(\Phi_{jj}\) and \(\Phi_{hh}\), we integrate out \(\Phi_{jh}\) over its support, denoted as \(\operatorname{supp}(\Phi_{jh})\), yielding:
	\begin{align}
		f_{\Phi_{jj}, \Phi_{hh}}(\phi_{jj}, \phi_{hh})
		= \frac{
			\exp\left(-\frac{1}{2} \left(\frac{\phi_{jj}}{d_j} + \frac{\phi_{hh}}{d_h}\right)\right)
		}{
			2^v (d_j d_h)^{v/2} \Gamma_2\left(\frac{v}{2}\right)
		}
		\mathlarger\int_{\operatorname{supp}(\Phi_{jh})} (\phi_{jj} \phi_{hh} - \phi_{jh}^2)^{\frac{v - 3}{2}} \, d\phi_{jh},
	\end{align}
	and conditional density of \(\Phi_{jh} \mid \Phi_{jj}, \Phi_{hh}\) is given by:
	\begin{align}\label{eq:conditional_pdf}
		f_{\Phi_{jh} \mid \Phi_{jj}, \Phi_{hh}}(\phi_{jh} \mid \phi_{jj}, \phi_{hh})
		\propto (\phi_{jj} \phi_{hh} - \phi_{jh}^2)^{\frac{v - 3}{2}}.
	\end{align}
	This conditional density is symmetric around 0 since
	\begin{align}
		f_{\Phi_{jh} \mid \Phi_{jj}, \Phi_{hh}}(-\phi_{jh} \mid \phi_{jj}, \phi_{hh}) = f_{\Phi_{jh} \mid \Phi_{jj}, \Phi_{hh}}(\phi_{jh} \mid \phi_{jj}, \phi_{hh}).
	\end{align}
	Given that \cref{eq:conditional_pdf} defines a probability density function, this symmetry implies the support of this conditional variable, denoted as \(\operatorname{supp}(\Phi_{jh} \mid \Phi_{jj}, \Phi_{hh})\), is also symmetric around 0. Therefore,
	\begin{align}
		\mathbb{E}[\Phi_{jh} \mid \Phi_{jj}, \Phi_{hh}]
		= \int_{\operatorname{supp}(\Phi_{jh} \vert \Phi_{jj}, \Phi_{hh})} \phi_{jh} \, f_{\Phi_{jh} \mid \Phi_{jj}, \Phi_{hh}}(\phi_{jh} \mid \phi_{jj}, \phi_{hh}) \, d\phi_{jh} = 0.
	\end{align}
	Since this holds for all \(\phi_{jj}, \phi_{hh} \ge 0\), it follows that
	\[
		\mathbb{E}[\Phi_{jh} \mid \Phi_{jj} = \lambda_{jj}, \Phi_{hh} = \lambda_{hh}] = 0 \quad \text{a.s.}.
	\]

	Now, recall that the Wishart distribution with diagonal scale matrix \(\bm{D}\) can be constructed from:
	\[
		\bm{\Phi} = \sum_{i=1}^v \bm{x}_i \bm{x}_i^\intercal,
		\quad \text{where } \bm{x}_i \sim \mathcal{N}_p(\bm{0}, \bm{D}) \text{ with independent components}.
	\]
	This implies
	\[
		\Phi_{jh} = \sum_{i=1}^v x_{ij} x_{ih}, \quad \Phi_{\ell\ell} = \sum_{i=1}^v x_{i\ell}^2.
	\]
	Thus, for \(j \neq h \neq \ell\), the independence of components implies \(x_{ij} \perp x_{ih} \perp x_{i\ell}\), which gives \(\Phi_{jh} \perp \Phi_{\ell\ell}\). Therefore,
	\begin{align}
		\mathbb{E}[\Phi_{jh} \vert \Phi_{11} = \lambda_{11}, \ldots, \Phi_{pp} = \lambda_{pp}] = 0,
	\end{align}
	leading to the matrix form:
	\[
		\mathbb{E}[\bm{\Phi} \mid \operatorname{diag}(\bm{\Phi}) = \bm{\Lambda}] = \bm{\Lambda},
	\]
	for any diagonal matrix \(\bm{\Lambda} = \operatorname{diag}(\lambda_{11}, \dots, \lambda_{pp})\) with \(\lambda_{jj} \ge 0\).

	Finally, applying Statement 1, define \(\bm{\Psi} = \bm{B} \bm{\Phi} \bm{B}^\intercal\). Then \(\bm{\Psi} \sim \mathcal{W}_p(v, \bm{B}^\intercal \bm{D} \bm{B})\), and \(\bm{B}\) corresponds to the eigenvectors of \(\bm{B} \bm{D} \bm{B}^\intercal\). Therefore,
	\begin{align*}
		\mathbb{E}[\bm{\Psi} \mid \mathcal{B} = \bm{B}, \bm{\Lambda}_{\bm{B}}(\bm{\Psi}) =  \bm{\Lambda}]
		 & = \mathbb{E}[\bm B \bm{\Phi} \bm B^\intercal \mid \mathcal{B} = \bm{B}, \bm\Lambda_{\bm B}\left(\bm B \bm{\Phi} \bm B^\intercal \right) =  \bm{\Lambda}]         \\
		 & = \mathbb{E}[\bm B \bm{\Phi} \bm B^\intercal \mid \mathcal{B} = \bm{B}, (\bm{B}^\intercal \bm{B} \bm{\Phi} \bm{B}^\intercal \bm{B}) \odot \bm{I} = \bm{\Lambda}] \\
		 & = \bm{B} \mathbb{E}[\bm{\Phi} \mid \operatorname{diag}(\bm{\Phi}) = \bm{\Lambda}] \bm{B}^\intercal                                                               \\
		 & = \bm{B} \mathbb{E}[\bm{\Phi} \mid \bm{\Lambda}_{\bm{B}}(\bm{\Psi}) = \bm{\Lambda}] \bm{B}^\intercal                                                             \\
		 & = \bm{B} \bm{\Lambda} \bm{B}^\intercal.
	\end{align*}
	This verifies the self-consistency condition stated in \cref{eq:pointwise_sc}.
\end{proof}

\begin{theorem}\label{sec4:prop3}
	Let $\mathcal P_{\mathcal B_j}, (j = 1,2,\ldots)$ denote the principal SPSD tensors from successive iterations of the self-consistency algorithm $K$-Tensors (\Cref{sec4:algo2}) for a random SPSD matrix $\bm\Psi \in \S_+^p$. Then $\E\Vert \bm \Psi - \mathcal P_{\mathcal B_{j}}(\bm\Psi) \Vert_{\F}^2$ is monotonically decreasing in $j$.
\end{theorem}

\begin{proof}
	Let \(\tilde{\mathcal P}_{\mathcal B_{j}}(\bm\Psi)\) denote the principal and self-consistent tensors defined by the conditional expectation \(\E[\bPsi \vert {\mathcal P}_{\mathcal B_{j}}(\bm\Psi)]\) on the $j$-th iteration.
	\begin{align}
		\E \left[ \left\Vert \bm\Psi - \mathcal{P}_{\mathcal{B}_{j}}(\bm\Psi) \right\Vert_{\F}^2 \right]
		 & \ge \E \left[ \left\Vert \bm\Psi - \E\left[\bm\Psi \,\big|\, \bm \Lambda_{\mathcal{B}_j}(\bm\Psi) \right] \right\Vert_{\F}^2 \right] \label{convergence:eq1}              \\
		 & = \E \left[ \left\Vert \bm\Psi - \tilde{\mathcal{P}}_{\mathcal{B}_{j}}(\bm\Psi) \right\Vert_{\F}^2 \right] \notag                                                         \\
		 & \ge \E \left[ \inf_{\mathcal{B} \in \mathcal V_p(\mathbb R^p)} \left\Vert \bm\Psi - \mathcal{P}_{\mathcal{B}}(\bm\Psi) \right\Vert_{\F}^2 \right] \label{convergence:eq3} \\
		 & = \E \left[ \left\Vert \bm\Psi - \mathcal{P}_{\mathcal{B}_{j+1}}(\bm\Psi) \right\Vert_{\F}^2 \right]. \notag
	\end{align}
	Note that in \cref{convergence:eq1}, the inequality holds because the conditional expectation is the best approximation of $\bPsi$ given the domain of attraction defined by $\mathcal P_{\mathcal B_j}(\bPsi)$, as detailed in \cref{sec3:def:cpc}. The second inequality \cref{convergence:eq3} holds because the common principal components $\mathcal B_{j+1}$ minimize the expected distance to the SPSD matrix $\bPsi$ in the Frobenius norm under the domain of attraction defined by $\tilde{\mathcal P}_{\mathcal B_{j}}(\bm\Psi)$, as explained in \cref{sec3:domain}.
\end{proof}

\Cref{sec4:prop3} demonstrates that the algorithm monotonically reduces the loss function, and {thus} shows that the $K$-Tensors algorithm guarantees convergence to a local minimum. To augment the possibility of finding the global minimum, we employ a strategy used for $K$-means, iterating the algorithm for multiple initial starting values and choosing the solution that produces the minimum loss.

\section{Numerical Studies}

Our numerical studies are divided into two sections. The first section provides a visualization of SPSD matrices and their corresponding principal SPSD tensors, obtained from $K$-Tensors algorithm. The second section involves a comparative analysis of the $K$-Tensors algorithm against various distance metrics mentioned in the introduction section.

\subsection{Visualization}

For visualization, each SPSD matrix is represented as an ellipsoid in \(\mathbb{R}^p\), where the eigenvectors determine the orientation of the ellipsoid's axes, and the corresponding eigenvalues define the axis lengths (i.e., the eccentricities). The simulation setup is illustrated in \Cref{fig:simulation_psi}.
\begin{figure}[htbp!]
	\centering
	\begin{tikzpicture}[
			box/.style={rectangle, draw, rounded corners, align=center, minimum width=4.8cm, minimum height=1cm},
			subbox/.style={rectangle, draw, align=center, minimum width=3.8cm, minimum height=0.9cm},
			arrow/.style={->, thick},
			node distance=1.0cm and 1.5cm,
			font=\footnotesize
		]

		\node[subbox] (Bm) {Eigenvector Matrix \\ $\bm{B}_m =
				\begin{bmatrix}
					\cos(2\pi m / 7)  & \sin(2\pi m / 7) & 0 \\
					-\sin(2\pi m / 7) & \cos(2\pi m / 7) & 0 \\
					0                 & 0                & 1
				\end{bmatrix}$};

		\node[subbox, left=of Bm, xshift=-0.2cm] (Di) {Eigenvalues \\ $d_{i1} \sim \chi^2(10), \quad d_{i2} \sim \chi^2(3)$ \\
			$\bm{D}_i = \operatorname{diag}(d_{i1}, d_{i2}, 0)$};

		\node[subbox, right=of Bm, xshift=0.2cm] (Ei) {Noise Matrix \\ $\bm{E}_i \sim \mathcal{W}_3(\bm{I}, 10)$};

		\node[subbox, below=of Di, yshift=-0.2cm] (BDiBt) {Signal Part \\ $\bm{B}_m \bm{D}_i \bm{B}_m^\intercal$};

		\node[subbox, below=of Ei, yshift=-0.2cm] (BEiBt) {Noise Part \\ $\bm{B}_m \bm{E}_i \bm{B}_m^\intercal$};

		\node[box, below=of Bm, yshift=-3.8cm] (Psi) {Final SPSD Matrix \\
			$\bm{\Psi}_{mi} = \bm{B}_m (\bm{D}_i + \bm{E}_i) \bm{B}_m^\intercal$};

		\draw[arrow] (Bm.south) -- (Psi.north);
		\draw[arrow] (Di.south) -- (BDiBt.north);
		\draw[arrow] (Ei.south) -- (BEiBt.north);
		\draw[arrow] (BDiBt.south) -- (Psi.north west);
		\draw[arrow] (BEiBt.south) -- (Psi.north east);

		\node[above=0.4cm of Bm] {\textbf{For each group $m = 1, \dots, 3$}};

	\end{tikzpicture}
	\caption{Simulation of $\bm{\Psi}_{mi}$ using structured eigenvectors $\bm{B}_m$, random eigenvalues $\bm{D}_i$, and Wishart noise $\bm{E}_i$.}
	\label{fig:simulation_psi}
\end{figure}
\begin{figure}[htbp!]
	\centering     
	\begin{subfigure}[b]{0.49\linewidth}
		\centering
		\includegraphics[width=\linewidth]{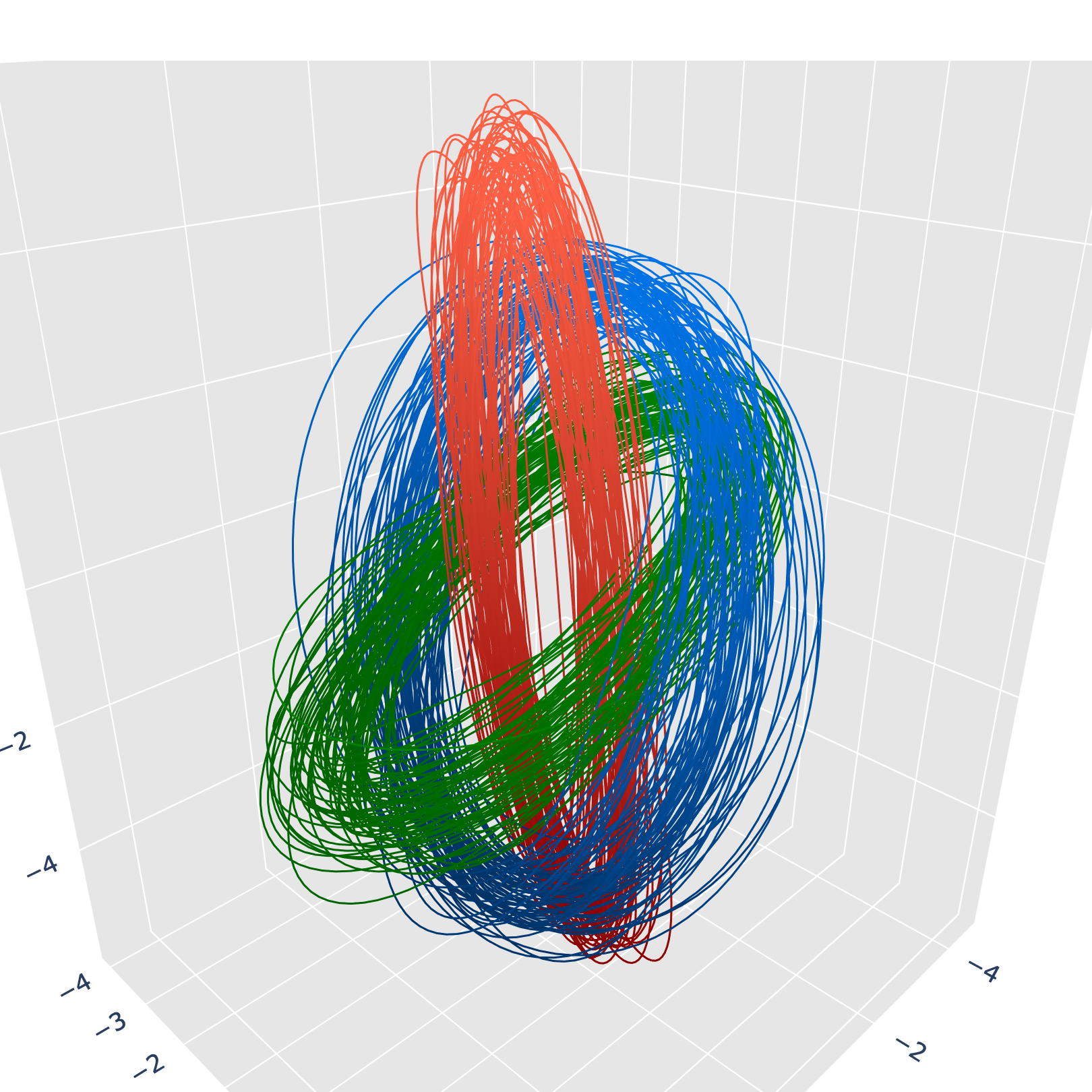}
		\caption{angel from the left side}
		\label{fig:2a}
	\end{subfigure}
	\begin{subfigure}[b]{0.49\linewidth}
		\centering
		\includegraphics[width=\linewidth]{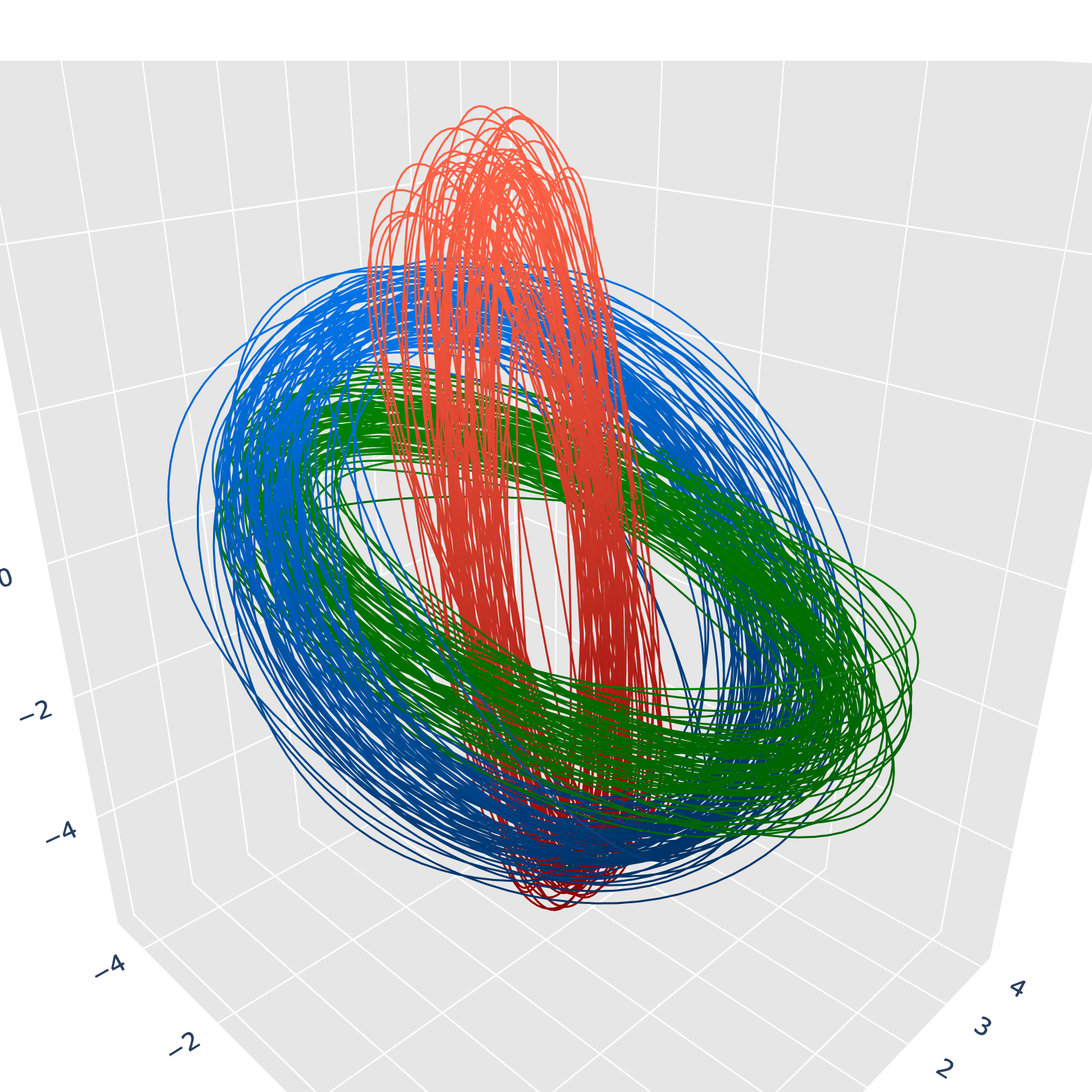}
		\caption{angel from the right side}
		\label{fig:2b}
	\end{subfigure}
	\begin{subfigure}[b]{0.49\linewidth}
		\centering
		\includegraphics[width=\linewidth]{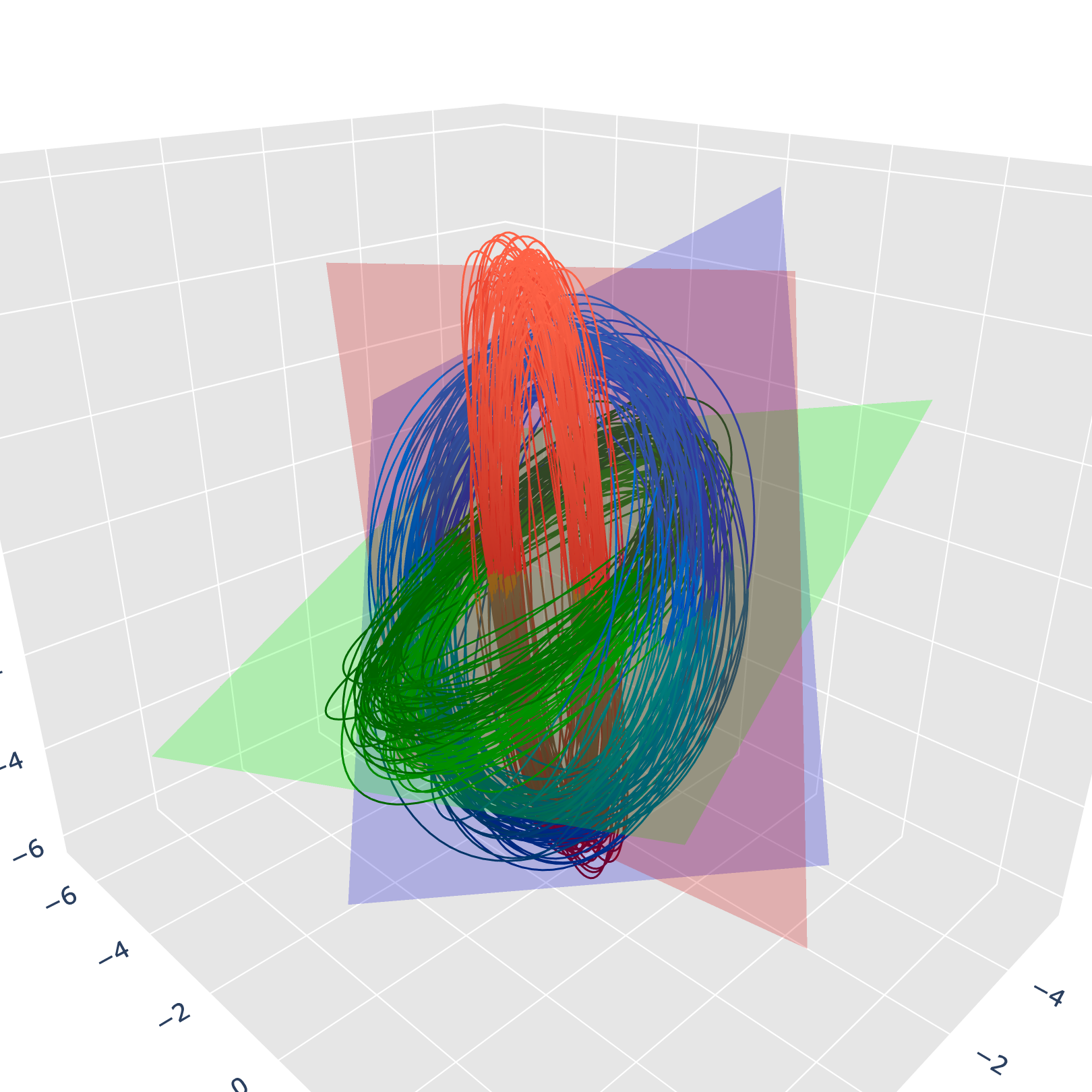}
		\caption{angel from the left side with plane}
		\label{fig:2c}
	\end{subfigure}
	\begin{subfigure}[b]{0.49\linewidth}
		\centering
		\includegraphics[width=\linewidth]{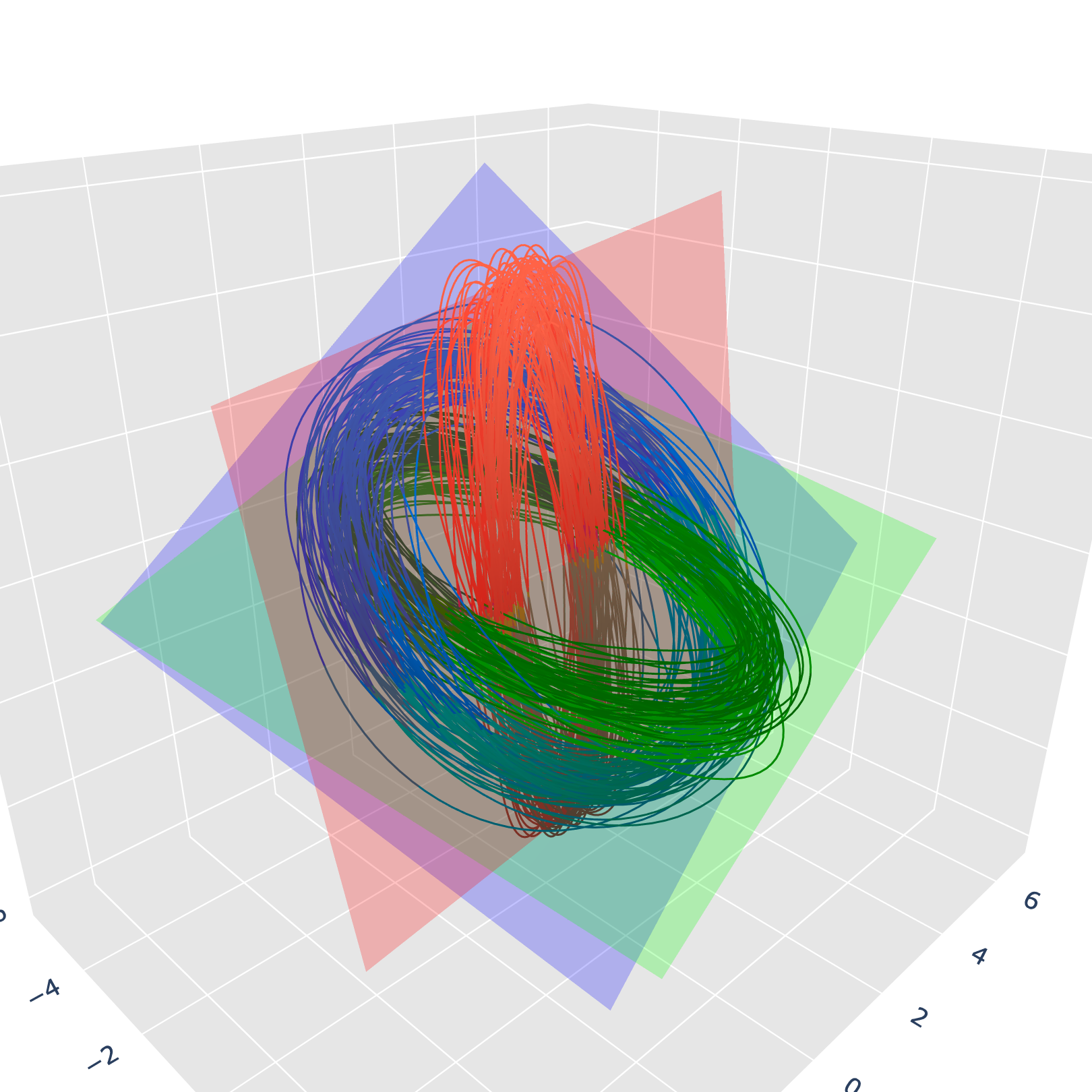}
		\caption{angel from the right side with plane}
		\label{fig:2d}
	\end{subfigure}
	\caption{Top row: Visualization of three groups represented as ellipsoids. Bottom row: Same visualization with principal SPSD tensors overlaid, illustrating the estimated subspaces for each group.}
	\label{fig2}
\end{figure}

\Cref{fig:2a} and \Cref{fig:2b} display the arrangement of ellipsoids in three-dimensional space from two different viewing angles, highlighting the spatial distribution of the simulated SPSD matrices. \Cref{fig:2c,fig:2d} present the same scene augmented with the principal SPSD tensors estimated via the $K$-Tensors algorithm. These tensors define the common principal component (CPC) subspaces for each group, visualized as color-matched planes centered within their respective clusters.

\subsection{Simulation-Based Evaluation}

\Cref{fig:simulate-cpc-tensor} provides a graphical illustration of the simulation process for each diagnostic group.
\begin{figure}[htbp]
	\centering
	\begin{tikzpicture}[
			box/.style={rectangle, draw, rounded corners, align=center, minimum width=4.5cm, minimum height=1cm},
			subbox/.style={rectangle, draw, align=center, minimum width=3.8cm, minimum height=0.9cm},
			arrow/.style={->, thick},
			node distance=1.0cm and 1.2cm,
			font=\footnotesize
		]

		\node[subbox] (Um) {Common Eigenvector Matrix $\bm{U}_m$ \\ $\bm{U}_m \in \mathbb{R}^{p \times p}$ (orthonormal)};
		\node[subbox, right=of Um] (chisq) {Eigenvalues \\ $d_j \sim \chi^2(\nu_j)$, sorted descending};

		\node[subbox, below=of chisq] (D) {Sample-specific \\ Diagonal Matrix $\bm{D}_{mi}$};
		\draw[arrow] (chisq.south) -- (D.north);

		\node[subbox, below=of Um] (Psi_clean) {Noise-free Tensor \\ $\bm{\Psi}^0_{mi} = \bm{U}_m \bm{D}_{mi} \bm{U}_m^\intercal$};
		\draw[arrow] (Um.south) -- (Psi_clean.north);
		\draw[arrow] (D.west) -- (Psi_clean.east);

		\node[box, below=1.8cm of Psi_clean] (Psi) {Final Tensor \\ $\bm{\Psi}_{mi} = \bm{\Psi}^0_{mi} + \bm{E}_{mi}$};
		\draw[arrow] (Psi_clean.south) -- (Psi.north);

		\node[subbox, right=2.0cm of Psi] (E) {Noise Matrix \\ $\bm{E}_{mi} \sim \text{Wishart}(p, \bm{I})$};
		\draw[arrow] (E.west) -- (Psi.east);

		\node[above=0.4cm of Um] {\textbf{For each generating group $m = 1, \dots, M$}};

	\end{tikzpicture}
	\caption{Graphical illustration of the simulation process for symmetric positive semi-definite (SPSD) tensors under a common principal component (CPC) model. Each generating group $m$ shares a common orthonormal basis $\bm{U}_m$, with sample-specific eigenvalues drawn from a chi-square distribution. Noise-free tensors are constructed via eigen-decomposition and then perturbed by Wishart-distributed noise.}
	\label{fig:simulate-cpc-tensor}
\end{figure}
For each generating group \( m = 1, \dots, M \), an orthonormal matrix \(\bm{U}_m \in \mathcal{V}_p(\mathbb{R}^p)\) is drawn to represent the shared eigenvectors within the group. Each \(\bm{U}_m\) is obtained via QR decomposition from a standard normal matrix:
\[
	\bm{A}_{ij} \sim \mathcal{N}(0, 1), \quad \bm{A} = \bm{Q} \bm{R}, \quad \bm{U}_m = \bm{Q}.
\]

\textbf{Sampling eigenvalues from chi-square distributions.}
A global degrees-of-freedom vector \(\bm{d} \in \mathbb{R}^p\) is sampled uniformly from \([0, 10]\) and sorted in decreasing order. For each observation \(i = 1, \ldots, 500\) within group \(m\), eigenvalues are independently drawn as \(d_{mij} \sim \chi^2(d_j)\), and assembled into diagonal matrices \(\bm{D}_{mi} = \mathrm{diag}(d_{mi1}, \ldots, d_{mip})\).

\textbf{Constructing noise-free covariance matrices.}
Each noise-free covariance matrix is constructed using the CPC structure:
\[
	\bm{\Psi}_{mi}^{(0)} = \bm{U}_m \bm{D}_{mi} \bm{U}_m^\top.
\]

\textbf{Adding noise via Wishart perturbation.}
Independent noise matrices \(\bm{E}_{mi}\) are sampled from a Wishart distribution:
\[
	\bm{E}_{mi} \sim \mathcal{W}_p(p, \bm{I}_p),
\]
where \(s > 0\) controls the noise scale. The final observed SPSD matrix is given by
\[
	\bm{\Psi}_{mi} = \bm{\Psi}_{mi}^{(0)} + \bm{E}_{mi}.
\]

\textbf{Simulation and evaluation.}
For each combination of simulated group number \(M \in \{2, 3, 4\}\) and matrix dimension \(p \in \{2, 5, 10\}\), 50 independent simulation runs are conducted. In each run, 500 SPSD matrices are generated for each of the \(M\) groups, yielding a total of \(500 \times M\) observations. Clustering is performed using four methods: the proposed \(K\)-Tensors algorithm and three \(K\)-means variants based on different distance metrics—Euclidean, affine-invariant Riemannian, and log-determinant divergence. In all methods, the number of clusters specified to the algorithm is set to \(K = M\), matching the number of underlying generating groups. Clustering performance is evaluated by the proportion of misclassified observations relative to the true group labels.

\begin{figure}[htp]
	\footnotesize
	\centering
	\includegraphics[width=\linewidth]{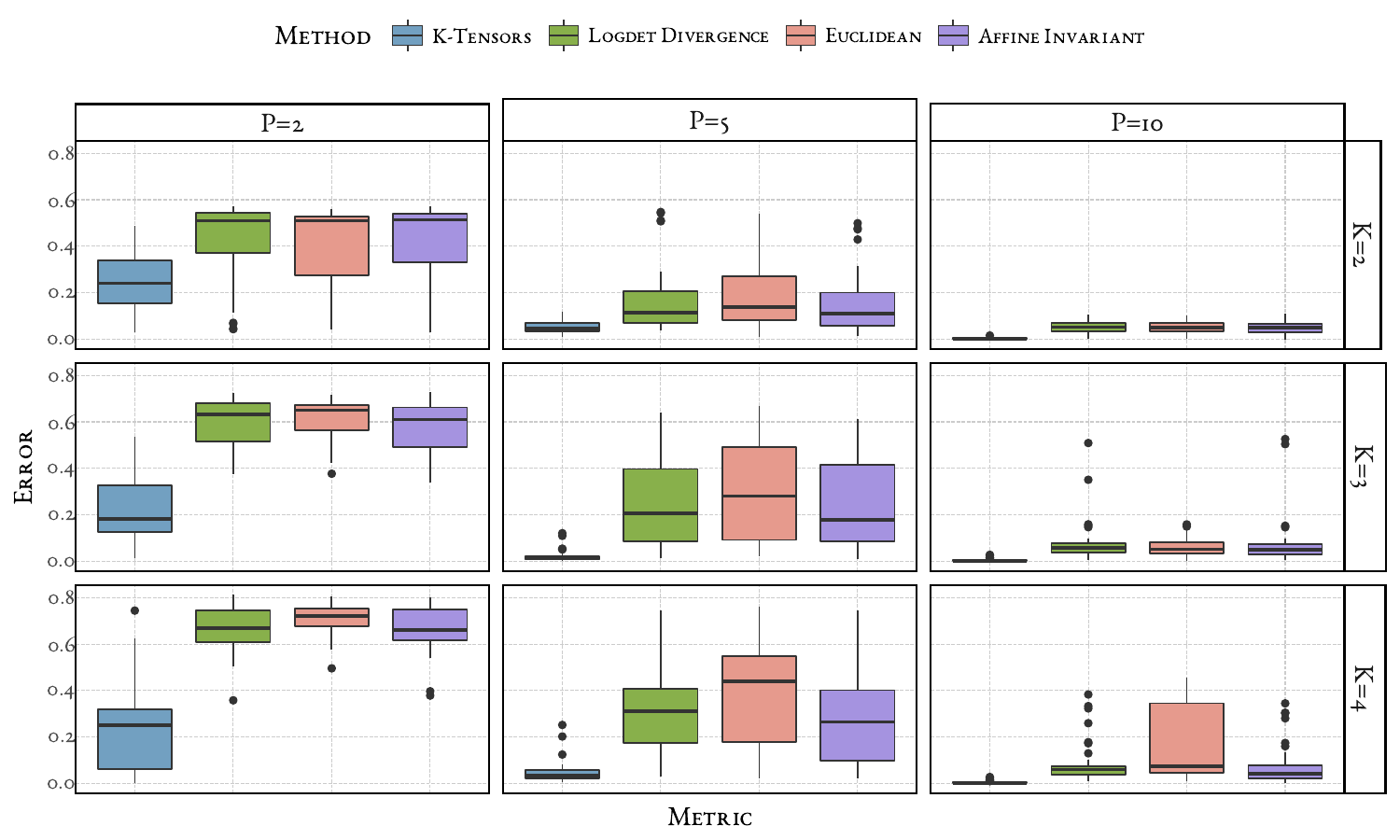}
	\caption{Boxplots of clustering error comparing $K$-Tensors with three $K$-means methods that use Euclidean, log-determinant, and affine-invariant distances, respectively. All results are based on the CPC tensor simulation setting.}
	\label{fig:ktensor-error-boxplot}
\end{figure}
\begin{table}[htp]
	\centering
	\footnotesize
	\begin{tabular}{c|cccc}
		\toprule
		\textbf{\(p\)} & \textbf{$K$-Tensors} & \textbf{Logdet}   & \textbf{Euclidean} & \textbf{Affine Inv.} \\
		\midrule
		\multicolumn{5}{c}{\textit{K = 2}}                                                                    \\
		\midrule
		2              & 0.25 (0.07, 0.47)    & 0.47 (0.27, 0.57) & 0.49 (0.30, 0.56)  & 0.47 (0.24, 0.56)    \\
		5              & 0.05 (0.02, 0.10)    & 0.15 (0.03, 0.50) & 0.18 (0.04, 0.49)  & 0.15 (0.04, 0.47)    \\
		10             & 0.00 (0.00, 0.01)    & 0.05 (0.01, 0.09) & 0.06 (0.01, 0.10)  & 0.05 (0.01, 0.08)    \\
		\midrule
		\multicolumn{5}{c}{\textit{K = 3}}                                                                    \\
		\midrule
		2              & 0.22 (0.07, 0.48)    & 0.58 (0.40, 0.71) & 0.62 (0.45, 0.72)  & 0.58 (0.38, 0.70)    \\
		5              & 0.02 (0.00, 0.06)    & 0.24 (0.04, 0.52) & 0.30 (0.05, 0.56)  & 0.23 (0.04, 0.53)    \\
		10             & 0.01 (0.00, 0.02)    & 0.07 (0.01, 0.15) & 0.07 (0.02, 0.14)  & 0.08 (0.01, 0.17)    \\
		\midrule
		\multicolumn{5}{c}{\textit{K = 4}}                                                                    \\
		\midrule
		2              & 0.23 (0.02, 0.61)    & 0.68 (0.54, 0.79) & 0.71 (0.61, 0.79)  & 0.66 (0.56, 0.78)    \\
		5              & 0.05 (0.01, 0.11)    & 0.32 (0.04, 0.62) & 0.40 (0.07, 0.67)  & 0.28 (0.05, 0.65)    \\
		10             & 0.01 (0.00, 0.02)    & 0.07 (0.01, 0.29) & 0.15 (0.02, 0.43)  & 0.07 (0.01, 0.30)    \\
		\bottomrule
	\end{tabular}
	\caption{Mean classification error with 5\% and 95\% quantiles for each method, grouped by number of clusters \(K\) and dimension~\(p\).}
	\label{tab:mean_error_metrics}
\end{table}
\Cref{tab:mean_error_metrics} and \Cref{fig:ktensor-error-boxplot} summarize the clustering errors across the four methods under each simulation setting. The $K$-Tensors algorithm consistently outperforms the other distance-based clustering approaches in terms of classification accuracy. Moreover, its runtime is comparable to that of standard $K$-means and substantially lower than methods based on affine-invariant Riemannian distance or log-determinant divergence, which require computationally expensive matrix inversions.

\section{fMRI Analysis Example}

The \textit{Human Connectome Project} (HCP) is a large-scale neuroimaging initiative funded by the National Institutes of Health (NIH) with the aim of mapping human brain connectivity using advanced imaging technologies and analytic methods \citep{VanEssen2013}. The HCP provides high-resolution structural and functional magnetic resonance imaging (MRI) data, along with extensive behavioral and demographic information, from a large cohort of healthy individuals. The project's flagship S1200 release includes imaging and behavioral measures from 1003 healthy young adults. Each participant completed four resting-state functional magnetic resonance imaging (rs-fMRI) runs, each approximately 15 minutes in duration with 1200 time points. The imaging data were preprocessed by the HCP using a standardized minimal preprocessing pipeline that includes spatial normalization, motion correction, and ICA-based denoising. In addition, HCP applied group-level spatial Independent Component Analysis (ICA) to define a data-driven parcellation into 15 components, referred to as Nodes 1--15. To account for temporal autocorrelation, signal thinning was performed based on effective sample size (ESS), and the resulting time series were used to compute subject-level covariance matrices for downstream analysis.

\subsection{Implementation Details}

The $K$-Tensors algorithm was applied to the 15x15 covariance matrices derived from rs-fMRI data. The algorithm was configured to run with a maximum of 1000 iterations and 10 random initializations, with the number of clusters set to 2. To evaluate the clustering results, we employed t-tests for continuous variables and Fisher's exact tests for categorical variables to identify significant differences between the two clusters. These differences highlight the biomedical relevance of the clustering results, demonstrating the effectiveness of our method in uncovering meaningful patterns in the data.


\subsection{Results}\label{sec5:results}
Using $K$-Tensors, the analysis revealed 148 factors with significant differences across multiple domains between the two clusters. Key findings, as summarized in \cref{tab: two-group}, span cognitive, emotional, motor, and neuroanatomical dimensions. These results are further contextualized with biomedical interpretations and supported by relevant literature to highlight their significance. \Cref{tab: two-group} summarizes the representative factors with significant differences between the two clusters, and \Cref{fig:cluster-dumbbell-rawmeans} provides a graphical illustration of the clustering results.
\begin{table}[htp]
	\footnotesize
	\begin{center}
		\setlength{\tabcolsep}{1.8mm}{
			\begin{tabular}{l c c c c r}
				\hline
				\hline
				factors                                  & mean 1   & mean 2   & t stat & p-value             & category                            \\
				\hline
				CardSort AgeAdj                          & 101.108  & 103.738  & -4.232 & $2.52\times10^{-5}$ & Cognitive Flexibility               \\
				Flanker AgeAdj                           & 101.319  & 102.692  & -2.174 & 0.030               & Executive Function                  \\
				ReadEng AgeAdj                           & 106.033  & 108.359  & -2.490 & 0.013               & Language Decoding                   \\
				PicVocab AgeAdj                          & 108.260  & 110.460  & -2.287 & 0.022               & Language Comprehension              \\
				ProcSpeed AgeAdj                         & 102.210  & 105.585  & -2.683 & 0.007               & Processing Speed                    \\
				CogFluidComp AgeAdj                      & 104.364  & 107.740  & -3.116 & 0.002               & Cognition Fluid Composite           \\
				CogTotalComp AgeAdj                      & 111.733  & 115.968  & -3.297 & 0.001               & Cognition Total Composite Score     \\
				CogEarlyComp AgeAdj                      & 105.726  & 108.579  & -2.853 & 0.005               & Cognition Early Childhood Composite \\
				ER40 CR                                  & 35.395   & 35.740   & -2.149 & 0.032               & Emotion Recognition                 \\
				FS L ChoroidPlexus Vol                   & 1150.817 & 1102.193 & 3.346  & 0.001               & Volume Segmentation                 \\
				FS R ChoroidPlexus Vol                   & 1295.551 & 1245.391 & 2.762  & 0.006               & Volume Segmentation                 \\
				FS L Parahippocampal Thck                & 2.697    & 2.741    & -2.609 & 0.009               & Surface Thickness                   \\
				FS R Parahippocampal Thck                & 2.671    & 2.721    & -3.684 & 0.0002              & Surface Thickness                   \\
				FS R Supramarginal Area                  & 3780.031 & 3693.4   & 2.461  & 0.014               & Surface Area                        \\
				FS R Insula Area                         & 2445.089 & 2399.899 & 2.425  & 0.0155              & Surface Area                        \\
				Emotion Task Acc                         & 97.101   & 97.850   & -3.274 & 0.001               & Emotion                             \\
				Emotion Task Median RT                   & 792.915  & 762.361  & 4.158  & 3.48E-05            & Emotion                             \\
				Language Task Story Avg Difficulty Level & 10.225   & 10.565   & -3.636 & 0.0003              & Language                            \\
				Language Task Median RT                  & 3589.825 & 3525.851 & 3.871  & 0.0001              & Language                            \\
				Relational Task Acc                      & 74.428   & 77.648   & -4.011 & 6.51E-05            & Relational                          \\
				Relational Task Median RT                & 1754.796 & 1704.865 & 2.540  & 0.011               & Relational                          \\
				Social Task Perc TOM                     & 48.643   & 49.855   & -2.240 & 0.025               & Social                              \\
				Social Task Median RT TOM                & 1047.898 & 993.933  & 2.907  & 0.004               & Social                              \\
				WM Task Acc                              & 85.627   & 88.055   & -4.199 & 2.91E-05            & Working Memory                      \\
				WM Task Median RT                        & 872.119  & 850.340  & 2.778  & 0.006               & Working Memory                      \\
				Endurance AgeAdj                         & 106.806  & 109.318  & -2.850 & 0.004               & Motor Endurance                     \\
				Dexterity AgeAdj                         & 99.165   & 101.229  & -3.312 & 0.001               & Motor Dexterity                     \\
				\hline
				\hline
			\end{tabular}}
		\caption{representative factors with significant difference in two clusters}
		\label{tab: two-group}
	\end{center}
\end{table}

\begin{figure}[htp]
	\centering
	\includegraphics[width = \textwidth]{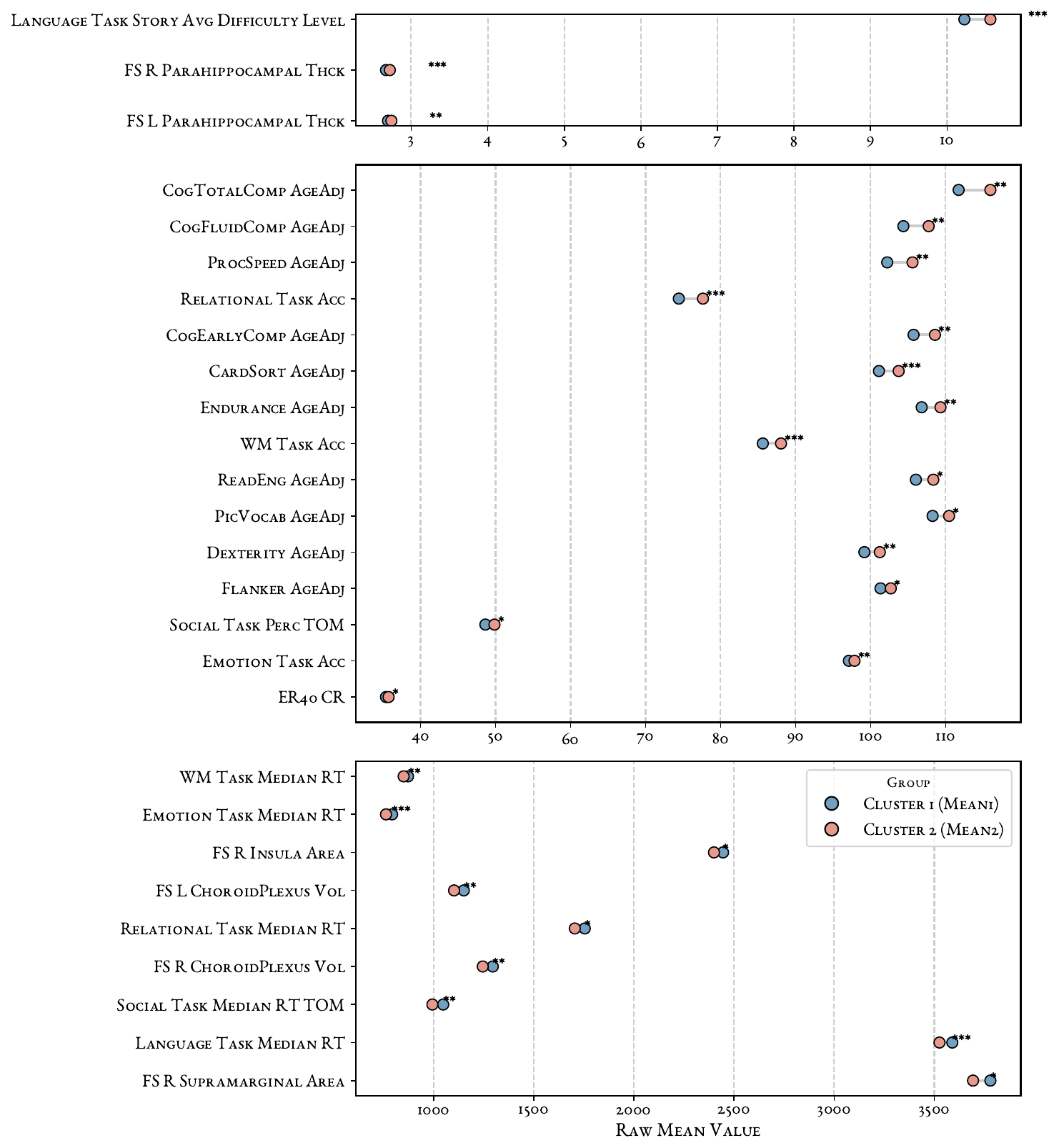}
	\caption{
		Dumbbell plots comparing raw mean values for representative factors with significant differences between the two clusters identified by $K$-Tensors.
		Stars denote significance: $^{*}p<0.05$, $^{**}p<0.01$, $^{***}p<0.001$.
	}
	\label{fig:cluster-dumbbell-rawmeans}
\end{figure}
In terms of cognitive performance, Cluster 2 demonstrated consistently higher scores than Cluster 1 across multiple cognitive domains, highlighting enhanced neuropsychological functioning. For instance, significant differences in Cognitive Flexibility (CardSort AgeAdj: $t=-4.232$, $p=2.52\times10^{-5}$) and Executive Function (Flanker AgeAdj: $t=-2.174$, $p=0.030$) suggest stronger prefrontal cortex activity in Cluster 2. Cognitive flexibility and executive function are heavily reliant on the dorsolateral prefrontal cortex, a region implicated in adaptive decision-making and inhibitory control \citep{miller2001integrative}. The superior Processing Speed (ProcSpeed AgeAdj: $t=-2.683$, $p=0.007$) and higher Fluid Cognition (CogFluidComp AgeAdj: $t=-3.116$, $p=0.002$) in Cluster 2 further reflect robust connectivity within large-scale networks, including the frontoparietal network. Such findings align with research demonstrating the role of the prefrontal and parietal cortices in supporting task-switching, working memory, and attentional control \citep{cole2013multi}.

For emotional and social task performance, {individuals in} Cluster 2 showed superior emotional processing capabilities, as indicated by improved accuracy in Emotion Recognition (ER40 CR: $t=-2.149$, $p=0.032$) and faster response times during emotional tasks (Emotion Task Median RT: $t=4.158$, $p=3.48\times10^{-5}$). Enhanced emotion recognition correlates with greater amygdala activity and its interaction with the ventromedial prefrontal cortex, regions critical for interpreting facial expressions and emotional cues \citep{adolphs2002neural}. In social cognition tasks, Cluster 2 displayed improved Theory of Mind (ToM) accuracy (Social Task Perc TOM: $t=-2.240$, $p=0.025$), with faster response times (Social Task Median RT TOM: $t=2.907$, $p=0.004$). These findings suggest superior mentalizing capabilities, potentially linked to better structural integrity or functional connectivity in the temporoparietal junction and medial prefrontal cortex, areas implicated in ToM processes \citep{saxe2013people}.

When it comes to motor function, {individuals in} Cluster 2 outperformed Cluster 1 in both Motor Endurance (Endurance AgeAdj: $t=-2.850$, $p=0.004$) and Dexterity (Dexterity AgeAdj: $t=-3.312$, $p=0.001$). These results align with findings that better motor performance correlates with increased gray matter volume in motor-related regions, such as the primary motor cortex and supplementary motor area \citep{grogan2012structural}. Additionally, motor endurance is often associated with cerebellar efficiency, supporting fine motor coordination and sustained physical activity \citep{buckner2013cerebellum}.

Additionally, {individuals in} Cluster 2 exhibited significant differences in neuroanatomical features. Reduced Choroid Plexus Volume (FS L: $t=3.346$, $p=0.001$; FS R: $t=2.762$, $p=0.006$) in Cluster 2 may reflect a lower inflammatory state, as increased choroid plexus volume is often linked to heightened neuroinflammatory responses, potentially contributing to cognitive decline \citep{baruch2013cns}. Increased Parahippocampal Thickness (FS L: $t=-2.609$, $p=0.009$; FS R: $t=-3.684$, $p=0.0002$) in Cluster 2 is noteworthy, as this region plays a crucial role in memory encoding and spatial navigation. Thicker parahippocampal cortices are associated with better memory performance and lower susceptibility to neurodegenerative diseases, such as Alzheimer's disease \citep{killiany2002mri}. Variations in cortical Surface Area were also observed, with Cluster 2 showing reduced Right Supramarginal Area (FS R: $t=2.461$, $p=0.014$) and Right Insula Area (FS R: $t=2.425$, $p=0.016$). While reduced surface area in these regions might indicate neural efficiency, it could also reflect developmental or experiential differences impacting social-emotional processing and interoception \citep{critchley2004neural}.

For task-specific performance, Cluster 2 achieved higher accuracy and faster response times in multiple task-based measures. For example, improved Working Memory Accuracy (WM Task Acc: $t=-4.199$, $p=2.91\times10^{-5}$) and faster Language Task Reaction Times (Language Task Median RT: $t=3.871$, $p=0.0001$) suggest optimized short-term memory storage and retrieval, likely mediated by the superior dorsolateral prefrontal cortex functioning. Enhanced performance in Relational Reasoning Tasks (Relational Task Median RT: $t=2.540$, $p=0.011$) underscores the efficiency of prefrontal-parietal connectivity, critical for complex reasoning and decision-making \citep{krawczyk2012cognition}.

These results provide compelling evidence of meaningful differences between the two clusters across cognitive, emotional, motor, and neuroanatomical domains. {individuals in} Cluster 2 exhibit enhanced functionality and structural integrity in areas critical for high-level processing and adaptive behavior. These findings not only validate the clustering approach used but also align with existing literature, underscoring the relevance of these biomarkers in understanding individual differences in cognitive and behavioral performance. Future research could explore the implications of these findings for disease susceptibility, treatment response, and personalized intervention strategies.

\section{Disscusion}
We emphasize the importance of our clustering algorithm development, especially its relevance to regression models employing SPSD matrix-valued observations. This link underscores the algorithm's role in improving analytical methods using SPSD matrices.

Recent developments in modeling SPSD matrices as outcomes in a regression framework have incorporated the concept of CPC. Specifically, \citet{zhao2021principal} and \citet{zhao2021covariate} have proposed novel approaches that account for SPSD matrix outcomes using a 1-dimensional quadratic form that assumes a shared subspace in the log-regression setting. Consider a regression setting where the outcomes are SPSD matrices, $\bm \psi_i = \text{cov}(\mathbf Y_i)$ for $i = 1,2,\ldots,n$, of a random vector $\mathbf Y_i \in \Re^p$, assuming $\mathbf Y_i$ is centered at 0 without loss of generality:
\begin{align*}
	\log\left\{ \left( \bm\gamma^\intercal \bm\psi_i \bm\gamma  \right) \right\} = \beta_0 + \bm x_i^\intercal \bm\beta_1
\end{align*}
In this heteroscedasticity model, the covariates collected from individual $i$ are denoted by $\bm x_i \in \Re^p$, where $\beta_0 \in \Re^1$ and $\bm \beta_1 \in \Re^{q-1}$ are the model parameters. The common subspace of SPSD matrices $\bm \psi_i$ is assumed to be spanned by a shared vector $\bm\gamma \in \Re^p$.

The magnitude of the direction spanned by the common orthonormal vector $\bm\gamma$ with respect to each covariance matrix is also a representation of the variance of the latent variable $\mathbf z_i = \mathbf Y_i \bm \gamma$, where $\text{cov}(\mathbf z_i) = \bm\gamma^\intercal \mathbf Y_i^\intercal \mathbf Y_i \bm\gamma = \bm\gamma^\intercal \bm \psi_i \bm\gamma$. This assumption enables the quadratic form $\bm\gamma^\intercal \bm\psi_i \bm\gamma$ to satisfy the heteroscedasticity model, with the covariance matrix for each observation being a SPSD matrix with a common subspace.

\citet{zhao2022longitudinal} extend this approach to account for longitudinally measured covariates for each subject, resulting in a more comprehensive and flexible method for analyzing complex data structures. However, these models, including tensor regression models, assume a common $\bm{\gamma}$ across all individuals. This assumption of a shared $\bm{\gamma}$ can obscure vital information within the eigenstructures, presenting an intriguing avenue for integrating \textit{principal positive semi-definite tensors} into regression models. This integration could potentially capture more nuanced relationships and variations among individual data points, enriching the model's analytical depth.

\citet{cook2008covariance} studied the sample covariance matrices of a random vector observed in different populations, assuming that the differences between the populations, aside from sample size, are solely due to the quadratic reduction $R(\bm\Psi) = \bm\alpha^\intercal \bm \Psi \bm\alpha$. Here, $\bm \alpha \in \Re^{p\times q}$ denotes a matrix with $q$ orthonormal columns. Subsequently, \citet{franks2019shared} extended this approach to the ``spiked covariance model,'' also known as the partial isotropy model, which was studied by \citet{mardia1979multivariate} and \citet{johnstone2001distribution}. This model is commonly used in multivariate statistical analysis and assumes that the covariance matrix of a high-dimensional random vector consists of a low-rank component (the ``spike'') and a high-rank component (the ``noise'').

Most statistical approaches have focused on data modalities producing outcomes in Euclidean spaces. However, many modern data modality methods produce samples of covariance matrices. The increasing degree to which matrix-valued data is being generated in applications has inspired our clustering strategy for positive semi-definite (SPSD) tensor data. Our clustering approach for SPSD matrices respects the underlying geometry of the data and opens up a pathway for further advances by incorporating sparsity and lower rank estimation into the subspace of the $K$-Tensors clustering centroids. Such enhancements aim to refine and optimize the clustering process, potentially improving the accuracy and interpretability of the resulting clusters.

\bibliography{references}

\end{document}